\documentclass{article}
\usepackage[utf8]{inputenc}
\usepackage{authblk}
\usepackage{setspace}
\usepackage[margin=1.25in]{geometry}
\usepackage{graphicx}
\graphicspath{ {./figures/} }
\usepackage{subcaption}
\usepackage{amsmath}
\usepackage{amssymb}
\usepackage{mathtools}
\usepackage{amsthm}
\usepackage{lineno}
\usepackage{makecell} 

\theoremstyle{plain}
\newtheorem{theorem}{Theorem}[section]

\theoremstyle{definition}

\theoremstyle{remark}
\newtheorem{remark}{Remark}
 
\title{Identifiable Representation and Model Learning for Latent Dynamic Systems}

\author[1]{Congxi Zhang}
\author[1*]{Yongchun Xie} 

\affil[1]{Beijing Institute of Control Engineering, Beijing, P.~R.~China.} 
\affil[*]{Address correspondence to: xieyongchun@vip.sina.com} 

\date{}

\begin{document}

\maketitle

\begin{abstract}
Learning identifiable representations and models from low-level observations is helpful for an intelligent spacecraft {to complete downstream tasks reliably}. For temporal observations, to ensure that the data generating process is provably inverted, most existing works either assume the noise variables in the dynamic mechanisms are (conditionally) independent or require that the interventions can directly affect each latent variable. However, in practice, the relationship between the exogenous inputs/interventions and the latent variables may follow some complex deterministic mechanisms. In this work, we study the problem of identifiable representation and model learning for latent dynamic systems. The key idea is to use an inductive bias inspired by controllable canonical forms, {which are sparse and input-dependent} by definition. We prove that, for linear and affine nonlinear latent dynamic systems  {with sparse input matrices}, it is possible to identify the latent variables up to scaling and determine the dynamic models up to some simple transformations. The results have the potential to provide some theoretical guarantees for developing more trustworthy decision-making and control methods for intelligent spacecrafts.
\end{abstract}


\section{Introduction}

	When executing tasks in complex environments, an intelligent spacecraft \cite{wu2022centroidal,yuanli,sun2023satellite} may encounter low-level observations such as images or outputs of pre-trained neural networks \cite{lilinfeng, jiang2022progress}. These observations generally do not have any physical meaning, and their relationship is usually highly nonlinear and unknown. But there may exist some high-level latent variables (or latents for short), which can represent these observations, and some invariant mechanisms that can describe the relationship between these latents and exogenous inputs \cite{bengio2013representation,scholkopf2021toward}. Then, these systems can be regarded as latent dynamic systems \cite{watter2015embed,gelada2019deepmdp,hafner2019learning,hafner2019dream}.

  Recovering these latents and estimating the corresponding dynamic models are crucial for an intelligent spacecraft to make trustworthy decisions and finish downstream tasks in rich observation environments \cite{bengio2013representation,tomar2023learning,wang2022research}. Generally, we expect
 the learned representations to be identifiable, i.e., it is identical to the latents which reflects the actual data generating process up to some simple transformations  \cite{locatello2019challenging, khemakhem2020variational}. Since learning identifiable representations purely from the observational data without suitable inductive bias or extra information is impossible  \cite{locatello2019challenging}, some works use side information to help identify the latents under the assumption that they are (conditionally) independent  \cite{khemakhem2020variational, kivva2022identifiability}. However, in practice, variables usually follow some statistic or dynamic mechanisms, which makes them causally related \cite{scholkopf2021toward, trauble2021disentangled}. In this case, both the latent variables and the mechanisms have to be recovered \cite{scholkopf2021toward, yang2021causalvae, shen2022weakly, brehmer2022weakly, ahuja2022weakly}.
 
 This work mainly focuses on representation learning of latents under dynamic mechanisms. Observations in dynamic processes are usually more complex, but they also provide us with more information \cite{yao2021learning, lippe2022citris, lachapelle2022disentanglement}. {To identify the latents in dynamic processes, most works are based on temporal independent component analysis (ICA), and the noises in the dynamic mechanisms are regarded as the independent components \cite{yao2021learning, yao2022temporally, lachapelle2022disentanglement, liu2023learning}.} This actually means that the agents should go through sufficiently diverse mechanisms in each state to fit distributions of noises which exactly correspond  to each mechanism. Requiring sufficiently diverse noises/mechanisms makes these methods useful for causal discovery where observations generated by different dynamic models are witnessed. However, this is not for learning a specific dynamic model, especially when the model contains deterministic mechanisms. While deterministic mechanisms do not always matter for representation learning of static data \cite{liang2023causal, von2023nonparametric}, since we can simply eliminate corresponding factors of variation, it is not the case for dynamic processes because deterministic mechanisms also determine the time evolution of a system \cite{scholkopf2021toward}. Actually, even if the deterministic mechanisms in a dynamic process are known in advance, we have little hope of getting very strong identifiable results without further assumptions or other kinds of supervision \cite{ahuja2021properties}. 
 
 One possible way to overcome this problem is to use the mechanism-based perspective developed in \cite{ahuja2021properties, ahuja2022weakly}. Ahuja et al.\cite{ahuja2021properties} shows that if the true mechanism is known, then the latents can be identified up to any
 equivariances of the mechanism. Based on this framework, Ahuja et al.\cite{ahuja2022weakly} proved that if one can perturb each latent variable, the latents can be identified up to permutation and scaling. The assumption that the interventions/inputs directly affect each latent variable also means that the latent dynamic system behaves like (nonlinear) first-order discrete integrators. However, in practice, the relationship between the inputs and the latent variables may be more complex
 than that, and one input may affect several latent variables as time goes on.  
 
 {To reliably recover the controllable latents and their dynamic models from observations, in this paper, we study the identifiable representation and model learning problem and extend the results in \cite{ahuja2022weakly} to (high-order) latent dynamic systems.} Since the representation of such a system itself is not unique, we consider using an inductive bias inspired by control canonical forms. The controllable canonical form is a standard representation method for dynamic systems, and it represents variables related to each input in a chain form  \cite{rugh1996linear,luenberger1979dynamic}. This representation is sparse by definition and reveals the relationship between the inputs and the latent variables. We summarize the key contributions of this paper as follows.
 \begin{itemize}
 	\item We show that it is possible to jointly identify the latents and the corresponding (high-order) dynamic models for latent dynamic systems  {with sparse input matrices}. To the best of our knowledge, this is the first identifiable representation learning result for latents under deterministic high-order dynamic mechanisms.
 	\item For the learned representations, while most works identify the latents under nonlinear mechanisms up to nonlinear component-wise invertible transformations, our results identify the latents up to scaling even when the dynamic system is affine nonlinear.
 	\item For the learned dynamic models, our results show that for single-input linear systems, the coefficients in the system matrix are identified to the truth value. For multi-input linear systems and affine nonlinear systems, the coefficients or coefficient functions are identified up to some trivial transformations that will not affect the one-step prediction.
 \end{itemize}
 
 \section{Related Work}  
 
 \subsection{Representation Learning in Reinforcement Learning} 
 In reinforcement learning, representation learning is usually used to abstract low-dimensional representations from high-dimensional observations to promote the sample efficiency and generalization ability \cite{tomar2023learning,liu2023learning,hafner2019learning}. Approaches include reconstruction based methods  \cite{jaderberg2016reinforcement,gelada2019deepmdp,hafner2019learning,hafner2019dream,lei2022causal,wu2023daydreamer}, contrastive learning  \cite{laskin2020curl}, self-predictive representations  \cite{schwarzer2020data}, metric-based methods  \cite{zhang2020learning}, and so on. However, few works consider the identifiability of the learned representations and models. Hence, it is unclear whether the true data generating process is reliably recovered.

 \subsection{Identifiable Representation Learning for Dynamic Processes}
 \label{sec Identifiable Representation Learning for Dynamic Processes} 
 For dynamic processes, identifiable representation learning methods try to invert
 the data generating process from time series data.
 
  {Most works utilize independent components to identify the latents \cite{hyvarinen2016unsupervised, hyvarinen2017nonlinear, halva2020hidden, hyvarinen2023nonlinear}. The noise variables are regarded as the independent components and are assumed to be (conditional) independent.} SlowVAE \cite{klindt2020towards} assume the noise variables follow a factorized generalized Laplace distribution. LEAP \cite{yao2021learning} leverage nonstationarity noises to identify representations and models for causal latent processes. CITRIS \cite{lippe2022citris} assumes the noise variables are mutually independent and use interventions to identify multidimensional causal factors up to their minimal causal variables. TDRL \cite{yao2022temporally} assumes that the noise variables are mutually independent and identify the latents up to an invertible, component-wise transformation. iCITRIS \cite{lippe2022icitris} takes both instantaneous and temporal effects into consideration. Lachapelle et al.\cite{lachapelle2024nonparametric} introduce a novel principle called mechanism sparsity regularization and recover the latents by regularizing the learned causal graph
 to be sparse. {However, relying on noise makes these methods unsuitable for deterministic mechanisms. Besides, for the dynamic systems, it is unclear whether the learned mechanisms determine the system's evolution or the generation process of the measurement noise.}
 
 {Another line of work mainly leverages contrastive pairs of observations to identify the latents.} Locatello et al.\cite{locatello2020weakly} shows that identifiability is possible if the latent variables are independent and pairs of observations are accessible. Brehmer et al.\cite{brehmer2022weakly} allows the latent variables to be causally related and use counterfactual data to identify both the representations and the causal graphs. Ahuja et al.\cite{ahuja2022weakly} leverages sparse perturbations to identify the latent variables that follow any unknown continuous distribution. Recently, some works relax the assumption that paired observations are attainable and instead use interventional data \cite{ahuja2023interventional, squires2023linear, von2023nonparametric, zhang2023identifiability, varici2023general}. {However, they generally also assume that the interventions have direct effects on each latent variable, while our results allow the latent variables and the exogenous inputs to follow high-order mechanisms.}
 
 {We summarize some attributes of different methods in the table below.
 
 \begin{table}[h]
 	\caption{Comparisons of our results to prior works on identifiable representation learning for dynamic processes. A check indicates the methods can be applied to this setting.}
 	\label{table 0}
 	\centering
 	\begin{tabular}{ccccc}
 		\hline
 		\makecell{Method} & \makecell{Deterministic \\ Mechanisms} & \makecell{Exogenous \\ Inputs/Interventions} & \makecell{High-order\\ Mechanisms} & Identifiability \\
 		\hline
 		\cite{yao2021learning,yao2022temporally} & $\times$ &$\times$ &$\checkmark$ & componentwise \\ 
 		\cite{lippe2022citris,lippe2022icitris} & $\times$ &$\checkmark$ &$\checkmark$ & componentwise/blockwise \\ 
 		\cite{lachapelle2024nonparametric} & $\times$ &$\checkmark$ &$\checkmark$ & $a$-consistency/$z$-consistency \\  
 		\cite{ahuja2022weakly} & $\checkmark$ &$\checkmark$ &$\times$ & permutation and scaling \\
 		\textbf{Ours} & $\checkmark$ &$\checkmark$ &$\checkmark$ & scaling\\ 
 		\hline
 	\end{tabular}  
 \end{table}
}
  \subsection{Controllable Canonical Forms}
  {Controllable canonical form is a standard representation form of dynamic systems. It represents variables related to the input in a chain form that { is sparse and input-dependent by definition}. However, in control theory, we directly consider high-level variables instead of their low-level observations, and establishing controllable canonical forms requires that the mechanisms are perfectly known. We will formulize this form in the next section, and more information about controllable canonical forms can be found in, e.g.,  \cite{rugh1996linear, luenberger1967canonical, luenberger1979dynamic}.}
  
\section{Materials and Methods}
  \subsection{Representation and Model Learning for Single-Input Latent Dynamic Systems}
 \label{Section Single-Input Latent Dynamic Systems}
 In this section, we provide the representation and model learning methods and identifiability results for linear single-input systems and affine nonlinear single-input systems. We discuss how to extend these results to corresponding multi-input cases in the next section.
 
 \subsubsection{Linear Cases}
 \paragraph{Latent Dynamic Systems}
 We consider a general $n$-order single-input linear dynamic system
 \begin{equation}
 	q^{t+n}-a_{n}q^{t+n-1}-\cdots-a_{1}q^{t} = b u^{t},
 	\label{linear system with single-input plain}
 \end{equation}
 where $q^{t+n}, \cdots, q^{t} \in \mathbb{R}$ are state variables. $a_1, a_2, ..., a_n, b \in \mathbb{R}$ are constant coefficients. $u^{t} \in (-U,U) \subseteq \mathbb{R}$ is the input at time $t$, which can be set to any value in the admissible control set $(-U,U)$.

 If we let $\boldsymbol{z}^{t}=[q^{t}, \cdots, q^{t+n-1}]^T$, then we can write system (\ref{linear system with single-input plain}) in its controllable canonical form as  
 \begin{equation}
 	\boldsymbol{z}^{t+1}=\boldsymbol{A} \boldsymbol{z}^{t} + \boldsymbol{B} u^{t},
 	\label{linear system with single-input}
 \end{equation}
 where $\boldsymbol{z}^{t} \in  \mathbb{R}^n$ denotes the latent variables at time $t$, and the distribution of $\boldsymbol{z}^{t}$ is unknown. Throughout this work, we assume that the order of the dynamic system (i.e., $n$) is known or predetermined. Otherwise we search for the order during training.  Coefficient matrices $\boldsymbol{A}\in \mathbb{R}^{n \times n}$ and $\boldsymbol{B}\in\mathbb{R}^{n}$ in (\ref{linear system with single-input}) have following forms 
 \begin{equation}
 	\boldsymbol{A}=\begin{bmatrix}0&1&0&\cdots&0\\
 		0&0&1&\cdots&0\\
 		\vdots&\vdots&\vdots&\ddots&\vdots\\
 		0&0&0&\cdots&1\\
 		a_1&a_2&a_3&...&a_n\end{bmatrix}, \boldsymbol{B}=\begin{bmatrix}0\\0\\ \vdots \\0\\b\end{bmatrix}.
 	\label{linear A B}
 \end{equation}
 It is easy to see that determining the coefficient matrices $\boldsymbol{A}$ and $\boldsymbol{B}$ is equivalent to determining all the coefficients in system (\ref{linear system with single-input plain}). Hence, it is equivalent to identifying the dynamic model. 
 
 The main reason why we directly assume the system is represented in controllable canonical forms described above is that our objective here is not causal discovery but representation and dynamic model learning \cite{scholkopf2021toward}. That is, we are not trying to discover the invariant part from a set of mechanisms. Instead, we want to learn the representation and the corresponding dynamic model for one specific system/mechanism in a sparse and input-dependent form.
 
 Another advantage of using controllable canonical forms is that they can capture the indirect effect of the exogenous input. Most existing works use (conditional) independence/dependence to determine whether a latent variable is controllable \cite{wang2022denoised, huang2022action, liu2023learning}, which works for first-order systems but may not for high-order systems. 
 
 { A typical example is the mechanical system, which is also mentioned in \cite{liu2023learning}. If we let $d_\phi$, $v_\phi$, $m_\phi$, and $F$ to denote the position, velocity, mass,
 and exogenous force, respectively, then Newton's Second Law tells us:
 \begin{equation}
 	\left\{
 	\begin{array}{l}
 		d_\phi^{t+1}=d_\phi^{t}+v_\phi^{t} \notag \\ 
 		v_\phi^{t+1}=v_\phi^{t}+\frac{1}{m_\phi}\cdot F^{t}\notag.
 	\end{array}
 	\right.
 	\label {f=ma}
 \end{equation}  
 Force $F^t$ has a direct effect on the next-step velocity $v_\phi^{t+1}$, but not on the next-step position $d_\phi^{t+1}$. If we condition on the current position $d_\phi^{t}$ and the current velocity $v_\phi^{t}$, then the next-step position $d_\phi^{t+1}$ is invariant of the force $F^{t}$, which means the position is uncontrollable. However, obviously, the force can change the position, i.e., the position is actually controllable, which creates a contradiction.
 
 In controllable canonical forms, the mechanical system is represented as
 \begin{equation}
 	\boldsymbol{z}^{t+1}=\begin{bmatrix}
 		0& 1\\-1& 2
 	\end{bmatrix} \boldsymbol{z}^{t} + \begin{bmatrix}
 		0\\1
 	\end{bmatrix} F^{t}, 
 \end{equation}
 where $\boldsymbol{z}^t=[m_\phi d_\phi^t,m_\phi d_\phi^t+m_\phi v_\phi^t]^T$. It is easy to see that both $d_\phi^t$ and $v_\phi^t$ are controllable (All components of $z^t$ are regarded as controllable, instead of using an independence test.) and can be inferred from $z^t$ through a simple linear transformation.}
 
 We then consider the observations. Since the latent variable $\boldsymbol{z}^{t}$ is not observable, we have to recover it from observations. In control theory, we usually have a high-level observational variable $\boldsymbol{h}^t\in \mathbb{R}^{n_h}$, which is related to the latent variable $\boldsymbol{z}^{t}$ by a known mixing matrix $\boldsymbol{H}\in \mathbb{R}^{n_h \times n}$, i.e., $\boldsymbol{h}^t=\boldsymbol{H}\boldsymbol{z}^t$ (see, e.g., \cite{rugh1996linear, rajendran2023interventional}). But for latent dynamic systems, we consider low-level observational data $\boldsymbol{x}^{t} \in \mathcal{X} \subset \mathbb{R}^m$, $m\geq n$, and it is related to $\boldsymbol{z}^{t}$ as
 \begin{equation}
 	\boldsymbol{x}^{t}=g(\boldsymbol{z}^{t}),
 	\label{g}
 \end{equation}
 where $g$: $\mathbb{R}^n \to \mathcal{X}$ is unknown and assumed to be an injective and differentiable function. 
 
 Notice that we assume $g(\cdot)$ is an injective function of $\boldsymbol{z}^t$, instead of $\boldsymbol{q}^t$. This means that current observation contains all the information that determines how the system evolves at this time, which is similar to the notion of state measurement in control theory. A more general case is that $\boldsymbol{x}^t$ is an injective function of $\boldsymbol{h}^t$, which is related to $\boldsymbol{z}^t$ as $\boldsymbol{h}^t=\boldsymbol{H}\boldsymbol{z}^t$. Under this condition, we can use observations in a sliding time window to construct a new observation that is an injective function of $\boldsymbol{z}^t$, which is a standard procedure. 
 
 Our objective is to invert the nonlinear mixing function $g$ and determine the dynamic model (\ref{linear system with single-input}).

 \paragraph{Representation and Model Learning}
 
 We then discuss how to use learning methods to invert $g$ and determine the dynamic model (\ref{linear system with single-input}). To recover the latent variable, we train an encoder $f$: $\mathcal{X}  \to \mathbb{R}^n$ to invert $g$, i.e.,
 \begin{equation}
 	\boldsymbol{\hat z}^{t}=f(\boldsymbol{x}^{t}),
 	\label{f}
 \end{equation}
 where $\boldsymbol{\hat z}^{t} \in  \mathbb{R}^n$ is the estimation of $\boldsymbol{z}^{t}$, and we assume that $f$ is a differentiable function. Here we define $\tau\doteq f \circ g$, which describes
 the relationship between the real $\boldsymbol{z}^t$ and the estimated one $\boldsymbol{\hat z}^t$, i.e., $\boldsymbol{\hat z}^t = f \circ g(\boldsymbol{z}^{t}) = \tau(\boldsymbol{z}^{t})$.
 
 To estimate the dynamic system, we let $\boldsymbol{\hat A}$ and $\boldsymbol{\hat B}$ be the estimation of $\boldsymbol{A}$ and $\boldsymbol{B}$ respectively, and they have following forms
 \begin{equation}
 	\boldsymbol{\hat A}=\begin{bmatrix}0&1&0&\cdots&0\\
 		0&0&1&\cdots&0\\
 		\vdots&\vdots&\vdots&\ddots&\vdots\\
 		0&0&0&\cdots&1\\
 		\hat a_1&\hat a_2&\hat a_3&...&\hat a_n\end{bmatrix}, \boldsymbol{\hat B}=\begin{bmatrix}0\\0\\ \vdots \\0\\\hat b\end{bmatrix},
 	\label{linear hat A hat B}
 \end{equation}
 where $\hat a_1, \hat a_2, ..., \hat a_n, \hat  b \in \mathbb{R}$ are the estimated values of $a_1, a_2, ..., a_n, b$, respectively.
 
 The estimated coefficient $\hat b$ (and hence $\boldsymbol{\hat B}$) is set to arbitrary nonzero constant in advance. We learn the encoder $f$ and the estimated coefficient matrix $\boldsymbol{\hat A}$ by solving 
 \begin{equation}
 	\mathop{\min}_{\substack{\boldsymbol{\hat A}, f(\cdot)}}  \mathop{E}   [|| f(\boldsymbol{x}^{t+1})-\boldsymbol{\hat A} f(\boldsymbol{x}^{t})-\boldsymbol{\hat B} u^t||^2].
 	\label{minimizing}
 \end{equation} 
 The expectation is taken over all the states $\boldsymbol{x}^{t}\in \mathcal{X}$, generated by all $\boldsymbol{z}^{t} \in  \mathbb{R}^n$ through $g(\cdot)$, and the admissible control input $u^t \in \mathbb{R}$ , while the corresponding distributions $p_x$ and $p_u$ are unknown. This means we need to set the input to different values at different initial states. This is usually considered as counterfactual experiments in causal inference, but it is a common practice in system identification and control.
 
 Then we have the following identifiability result.

 \begin{theorem}
 	Suppose the latent dynamic system is represented as in (\ref{linear system with single-input}) and (\ref{g}). The coefficients satisfy $a_1 \neq 0$, $b \neq 0$, and $\sum_{i=1}^n a_i \neq 1$. We learn the representation and dynamic model by solving (\ref{minimizing}). Then $\tau(\boldsymbol{z})=\frac{\hat b}{b}\cdot \boldsymbol{z}$ and $\boldsymbol{\hat A}=\boldsymbol{A}$.
 	\label{Theorem linear single-input}
 \end{theorem}

 \begin{proof}
 	The proof contains two main steps:
 	\begin{itemize}
 		\item In step 1, we show that $\tau(\cdot)$ is an affine function by extending the Proposition 1 in \cite{ahuja2022weakly} to high-order dynamic systems. 
 		\item In step 2, we prove $\tau(\boldsymbol{z})=\frac{\hat b}{b}\cdot \boldsymbol{z}$ and $\boldsymbol{\hat A}=\boldsymbol{A}$ by utilizing the structures of $\boldsymbol{A}$, $\boldsymbol{\hat A}$, $\boldsymbol{B}$, and $\boldsymbol{\hat B}$.
 	\end{itemize}
 
 	\textbf{Step 1.}
 	By following the proof technique in \cite{ahuja2022weakly}, solving (\ref{minimizing}) implies
 	\begin{align}
 		f(\boldsymbol{x}^{t+1})&=\boldsymbol{\hat A} f(\boldsymbol{x}^{t})+\boldsymbol{\hat B} u ^{t} \notag \\
 		f \circ g(\boldsymbol{z}^{t+1})&=\boldsymbol{\hat A} f\circ g(\boldsymbol{z}^{t})+\boldsymbol{\hat B} u^{t} \notag \\
 		\tau(\boldsymbol{z}^{t+1})&=\boldsymbol{\hat A} \tau(\boldsymbol{z}^{t})+\boldsymbol{\hat B} u ^{t}.
 		\label{estimated dynamic equation}
 	\end{align}
 	Substituting (\ref{linear system with single-input}) into (\ref{estimated dynamic equation}) gives
 	\begin{equation} 
 		\tau(\boldsymbol{A} \boldsymbol{z}^{t} + \boldsymbol{B} u^{t})=\boldsymbol{\hat A} \tau(\boldsymbol{z}^{t})+\boldsymbol{\hat B} u ^{t}.
 		\label{u}
 	\end{equation}
 	Since $u^{t}$ can be set to any value in $(-U,U)$, we let $u^{t}=0$ to obtain
 	\begin{equation} 
 		\tau(\boldsymbol{A} \boldsymbol{z}^{t})=\boldsymbol{\hat A} \tau(\boldsymbol{z}^{t}).
 		\label{no u}
 	\end{equation}
 	Since $f$ and $g$ are assumed to be differentiable functions, $\tau=f \circ g$ is also a differentiable function. Taking the gradient of (\ref{u}) and (\ref{no u}) w.r.t $z^t$ gives 
 	\begin{equation}
 		\left\{
 		\begin{array}{l}
 			\boldsymbol{J}(\boldsymbol{A} \boldsymbol{z}^{t} + \boldsymbol{B} u^{t})\boldsymbol{A}=\boldsymbol{\hat A}  \boldsymbol{J}(z^{t}) \notag \\ 
 			\boldsymbol{J}(\boldsymbol{A} \boldsymbol{z}^{t})\boldsymbol{A}=\boldsymbol{\hat A}  \boldsymbol{J}(\boldsymbol{z}^{t})\notag,
 		\end{array}
 		\right.
 		\label {J}
 	\end{equation}
 	where $\boldsymbol{J}(\cdot)$ is the Jacobian of $\tau(\cdot)$. Then we have
 	\begin{equation} 
 		\boldsymbol{J}(\boldsymbol{A} \boldsymbol{z}^{t} + \boldsymbol{B} u^{t})\boldsymbol{A}=\boldsymbol{J}(\boldsymbol{A} \boldsymbol{z}^{t})\boldsymbol{A}.
 	\end{equation} 
 	The assumption $a_1 \neq 0$ implies that $\boldsymbol{A}$ is nonsingular. Since $\boldsymbol{A}$ is nonsingular and $\boldsymbol{z}^{t}$ varies within all the states, we obtain that for all $\boldsymbol{y} \in  \mathbb{R}^n$ and $v  \in (-U,U)$,
 	\begin{equation} 
 		\boldsymbol{J}(\boldsymbol{y} + \boldsymbol{B} v)-\boldsymbol{J}(\boldsymbol{y})=0.
 	\end{equation}

 	At time step $t+n$, by recursively using (\ref{linear system with single-input}) and (\ref{estimated dynamic equation}), we obtain
 	\begin{align}
 		\boldsymbol{z}^{t+n}=\boldsymbol{A}^n\boldsymbol{z}^{t}+\boldsymbol{A}^{n-1}\boldsymbol{B}u^{t}+\boldsymbol{A}^{n-2}\boldsymbol{B}u^{t+1} +...+\boldsymbol{B}u^{t+n-1},
 		\label{z tn}
 	\end{align}
 	and
 	\begin{align}
 		\tau(\boldsymbol{z}^{t+n})=\boldsymbol{\hat A}^n\tau(\boldsymbol{z}^{t})+\boldsymbol{\hat A}^{n-1} \boldsymbol{\hat B}u^{t}+\boldsymbol{\hat A}^{n-2}\boldsymbol{\hat B}u^{t+1}+...+\boldsymbol{\hat B}u^{t+n-1}.
 		\label{a tn}
 	\end{align}
 	Substituting (\ref{z tn}) into (\ref{a tn}) gives
 	\begin{align}
 		\tau(\boldsymbol{A}^n\boldsymbol{z}^{t}+\boldsymbol{A}^{n-1}\boldsymbol{B}u^{t}+...+\boldsymbol{B}u^{t+n-1}) 
 		=\boldsymbol{\hat A}^n\tau(\boldsymbol{z}^{t})+\boldsymbol{\hat A}^{n-1}\boldsymbol{\hat B}u^{t}+...+\boldsymbol{\hat B}u^{t+n-1}.
 		\label{a tn 2}
 	\end{align} 	
 	Since equation (\ref{a tn 2}) holds for all $u^{t},u^{t+1},...,u^{t+n-1}\in (-U,U)$, we let $u^{t}=u^{t+1}=,\cdots,=u^{t+n-1}=0$ to obtain
 	\begin{align}
 		\tau(\boldsymbol{A}^n\boldsymbol{z}^{t}) =\boldsymbol{\hat A}^n\tau(\boldsymbol{z}^{t}).
 		\label{no u tn 2}
 	\end{align}  
 	Taking the gradient of (\ref{a tn 2}) and (\ref{no u tn 2}) w.r.t $\boldsymbol{z}^t$ gives 	
 	\begin{align}
 		\boldsymbol{J}(\boldsymbol{A}^n\boldsymbol{z}^{t}+\boldsymbol{A}^{n-1}\boldsymbol{B}u^{t}+...+\boldsymbol{B}u^{t+n-1})\boldsymbol{A}^n =\boldsymbol{\hat A}^n \boldsymbol{J}(\boldsymbol{z}^{t}) =\boldsymbol{J}(\boldsymbol{A}^n\boldsymbol{z}^{t})\boldsymbol{A}^n.
 	\end{align} 
 	
 	Since $\boldsymbol{A}$ is nonsingular and $\boldsymbol{z}^{t}$ varies within all the states, we obtain $\forall \boldsymbol{y} \in  \mathbb{R}^n$ and $\forall u^{t},u^{t+1},...,u^{t+n-1}\in (-U,U)$,
 	\begin{equation} 
 		\boldsymbol{J}(\boldsymbol{y} + \boldsymbol{A}^{n-1}\boldsymbol{B}u^{t}+...+\boldsymbol{B}u^{t+n-1})-\boldsymbol{J}(\boldsymbol{y})=\boldsymbol{0}.
 		\label{Ja t n}
 	\end{equation} 
 	It is easy to verify that span\{${\boldsymbol{A}^{n-1}\boldsymbol{B},\boldsymbol{A}^{n-2}\boldsymbol{B},...,\boldsymbol{B}}$\}=$\mathbb{R}^n$. Then from (\ref{Ja t n}) we obtain that $\boldsymbol{J}(\boldsymbol{y})$ is a constant matrix for all $\boldsymbol{y} \in \mathbb{R}^n$. Then we conclude that $\tau(\cdot)$ is an affine function, i.e., 
 	\begin{equation} 
 		\tau(\boldsymbol{z})=\boldsymbol{T}\boldsymbol{z}+\boldsymbol{c},
 		\label{a affine}
 	\end{equation} 
 	where $\boldsymbol{T} \in \mathbb{R}^{n \times n}$ and $\boldsymbol{c} \in \mathbb{R}^{n}$ are both constant. 
 	
 	\textbf{Step 2.} Substituting (\ref{a affine}) into (\ref{u}) gives
 	\begin{align}
 		\boldsymbol{T}\boldsymbol{A}\boldsymbol{z}^{t}+\boldsymbol{T}\boldsymbol{B}u^{t}+\boldsymbol{c}=\boldsymbol{\hat A} \boldsymbol{T}\boldsymbol{z}^{t}+\boldsymbol{\hat A} \boldsymbol{c} +\boldsymbol{\hat B} u^{t},   \notag
 	\end{align} 
 	which yields 
 	\begin{equation}
 		\left\{
 		\begin{array}{l}
 			\boldsymbol{T}\boldsymbol{A} = \boldsymbol{\hat A} \boldsymbol{T} \notag\\
 			\boldsymbol{T}\boldsymbol{B} = \boldsymbol{\hat B} \notag\\
 			\boldsymbol{c}=\boldsymbol{\hat A} \boldsymbol{c}.
 		\end{array}
 		\right.
 	\end{equation}
 	$\boldsymbol{T}\boldsymbol{B}=\boldsymbol{\hat B}$ implies the last column of $\boldsymbol{T}$ is $[0,\cdots,0,\frac{\hat b}{b}]^T$. Then $\forall i=1,...,n-1$, left-multiply $\boldsymbol{T}\boldsymbol{B}=\boldsymbol{\hat B}$ by $\boldsymbol{\hat A}^i$ and use $\boldsymbol{T}\boldsymbol{A} = \boldsymbol{\hat A} \boldsymbol{T}$ to obtain 
 	\begin{equation} 
 		\boldsymbol{T}\boldsymbol{A}^{i}\boldsymbol{B} = \boldsymbol{\hat A}^{i}\boldsymbol{\hat B}.
 	\end{equation}
 	From $\boldsymbol{T}\boldsymbol{A}^{i}\boldsymbol{B} = \boldsymbol{\hat A}^{i}\boldsymbol{\hat B}$, after some calculation, we have 
 	\begin{equation}
 		\boldsymbol{T}=\begin{bmatrix}\frac{\hat b}{b}&0&0&0&0\\
 			*&\frac{\hat b}{b}&0&\cdots&0\\
 			\vdots&\vdots&\vdots&\ddots&\vdots\\
 			*&\cdots&*&\frac{\hat b}{b}&0\\
 			*&*&\cdots&*&\frac{\hat b}{b}\end{bmatrix},
 		\label{T}
 	\end{equation}
 	where $*$ denotes unimportant items. Then the equation $\boldsymbol{T}\boldsymbol{A} = \boldsymbol{\hat A} \boldsymbol{T}$ can be written as 
 	\begin{equation}
 		\begin{bmatrix}\frac{\hat b}{b}&0&0&0&0\\
 			*&\frac{\hat b}{b}&0&\cdots&0\\
 			\vdots&\vdots&\vdots&\ddots&\vdots\\
 			*&\cdots&*&\frac{\hat b}{b}&0\\
 			*&*&\cdots&*&\frac{\hat b}{b}\end{bmatrix} \begin{bmatrix}0&1&0&\cdots&0\\
 			0&0&1&\cdots&0\\
 			\vdots&\vdots&\vdots&\ddots&\vdots\\
 			0&0&0&\cdots&1\\
 			a_1&a_2&a_3&...&a_n\end{bmatrix}=\begin{bmatrix}0&1&0&\cdots&0\\
 			0&0&1&\cdots&0\\
 			\vdots&\vdots&\vdots&\ddots&\vdots\\
 			0&0&0&\cdots&1\\
 			\hat a_1&\hat a_2&\hat a_3&...&\hat a_n\end{bmatrix} \begin{bmatrix}\frac{\hat b}{b}&0&0&0&0\\
 			*&\frac{\hat b}{b}&0&\cdots&0\\
 			\vdots&\vdots&\vdots&\ddots&\vdots\\
 			*&\cdots&*&\frac{\hat b}{b}&0\\
 			*&*&\cdots&*&\frac{\hat b}{b}\end{bmatrix}.
 		\label{T}
 	\end{equation}
 	
 	The first column of both side of above equation implies the first column of $\boldsymbol{T}$ is $[\frac{\hat b}{b}, 0, \cdots, 0]^T$. Continuing this procedure to other columns gives that $\boldsymbol{T}=\frac{\hat b}{b}\boldsymbol{I}_n$, where $\boldsymbol{I}_n$ denotes n-dimensional identity matrix, and hence $\boldsymbol{\hat A}=\boldsymbol{A}$.

 	The assumption $\sum_{i=1}^n a_i \neq 1$ implies that $\boldsymbol{I}_n-\boldsymbol{A}$ is nonsingular. Then $\boldsymbol{c}=\boldsymbol{\hat A} \boldsymbol{c}= \boldsymbol{A} \boldsymbol{c}$ implies $\boldsymbol{c}=0$. Hence $\tau(\boldsymbol{z})=\frac{\hat b}{b}\boldsymbol{z}$.
 \end{proof}
 
 \begin{remark}
 	 Here, we discuss some main assumptions in Theorem \ref{Theorem linear single-input}. $a_1 \neq 0$ ensures that $\boldsymbol{A}$ is nonsingular, meaning all the states are reachable. If the input $u$ can be set to sufficiently large value, this assumption is not needed. $b \neq 0$ means that $u$ can affect this system. { These assumptions hold for many linear systems like spring-mass-damper systems or linear rigid body systems.} $\sum_{i=1}^n a_i \neq 1$ means the open-loop system is not marginal stable, i.e., when $u=0$, the equilibrium point of the system is 
 	unique. This unique equilibrium is naturally defined as the zero point.  { A frictionless rigid body system does not satisfy this assumption, but we can still derive a weaker identifiable conclusion.} In this case ($\sum_{i=1}^n a_i = 1$), an offset variable may exist, i.e., $\tau(\boldsymbol{z})=\frac{\hat b}{b}\cdot \boldsymbol{z} + \boldsymbol{c}$, where $\boldsymbol{c} \in \mathbb{R}^n$ is a constant vector. But the conclusion that $\boldsymbol{\hat A} = \boldsymbol{A}$ still holds.
 \end{remark}
 
 Theorem \ref{Theorem linear single-input} shows that, for the linear single-input system, we can identify the latents up to scaling. The only uncertainty comes from the estimation of $b$. The coefficients $a_1$, $\cdots$, $a_n$ are also identifiable, which is invariant to the estimation error of $b$.

 \subsubsection{Affine Nonlinear Cases}

 For affine nonlinear systems, we assume that the system is locally linearizable, i.e., the system is given as
 \begin{equation}
 	q^{t+n}-a_{n}(\boldsymbol{z}^{t})q^{t+n-1}-\cdots-a_{1}(\boldsymbol{z}^{t})q^{t} = b(\boldsymbol{z}^{t}) u^{t},
 	\label{nonlinear system with single-input plain}
 \end{equation}
 where $\boldsymbol{z}^{t}=[q^{t}, \cdots, q^{t+n-1}]^T$. All the coefficients are now nonlinear functions of $\boldsymbol{z}^t$ (and hence nonlinear functions of $q^{t}, \cdots, q^{t+n-1}$). The coefficient functions $a_{1}(\cdot)$, $\cdots$, $a_{n}(\cdot)$, and $b(\cdot)$ are all differentiable and bounded, and we assume that $|a_1(\cdot)|\geq h_a>0$, $|b(\cdot)|\geq h_b>0$, where $h_a$ and $h_b$ are arbitrary positive constants.

 we formalize the system (\ref{nonlinear system with single-input plain}) as:
 \begin{equation}
 	\boldsymbol{z}^{t+1}=\boldsymbol{A}(\boldsymbol{z}^{t})\boldsymbol{z}^{t} + \boldsymbol{B}(\boldsymbol{z}^{t}) u^{t},
 	\label{nonlinear system with single-input}
 \end{equation}
 where
 \begin{align}
 	&\boldsymbol{A}(\boldsymbol{z})=\begin{bmatrix}0&1&0&\cdots&0\\
 		0&0&1&\cdots&0\\
 		\vdots&\vdots&\vdots&\ddots&\vdots\\
 		0&0&0&\cdots&1\\
 		a_1(\boldsymbol{z})&a_2(\boldsymbol{z})&a_3(\boldsymbol{z})&...&a_n(\boldsymbol{z})\end{bmatrix} \notag \\
 	&\boldsymbol{B}(\boldsymbol{z})=\begin{bmatrix}0&0& \cdots &0&b(\boldsymbol{z})\end{bmatrix}^T.
 	\label{nonlinear A B}
 \end{align}
 
 The estimated matrices $\boldsymbol{\hat A}(\boldsymbol{z})$ and $\boldsymbol{\hat B}(\boldsymbol{z})$ have the same structures as in (\ref{nonlinear A B}), but the coefficient functions $a_1(\cdot)$, $a_2(\cdot)$, ..., $a_n(\cdot)$, and $b(\cdot)$ are replaced by their estimates $\hat a_1(\cdot)$, $\hat a_2(\cdot)$, ..., $\hat a_n(\cdot)$, and $\hat b(\cdot)$. The estimated coefficient functions are assumed to be differentiable. 
 
 The mixing function $g$ and the learned encoder $f$ are the same as in (\ref{g}) and (\ref{f}). The representation learning model is learned by solving 
 \begin{align}
 	\mathop{\min}_{\substack{\boldsymbol{\hat A}(\cdot), \boldsymbol{\hat B}(\cdot), f(\cdot)}}&  \mathop{E} [|| f(\boldsymbol{x}^{t+1})- \boldsymbol{\hat A}(f(\boldsymbol{x}^{t})) f(\boldsymbol{x}^{t})-\boldsymbol{\hat B}(f(\boldsymbol{x}^{t})) u^t||^2]\notag \\
 	\text{s.t.}\quad \ \ \ &  |\hat b(\cdot)|\geq \hat h_b>0.
 	\label{minimizing3}
 \end{align}
 Since the coefficient function $b(\cdot)$ is now a nonlinear function, we have to learn $\hat b(\cdot)$ instead of specifying it in advance. To prevent learning a trivial solution which maps each state to zero, we require $|\hat b(\cdot)|$ to be equal or greater than arbitrary positive constant $\hat h_b$.

 Then we have the following identifiability result.
 
 \begin{theorem}
 	Suppose the latent dynamic system is represented as in (\ref{nonlinear system with single-input}) and (\ref{g}). The coefficient functions $a_1(\cdot), \cdots, a_n(\cdot)$, and $b(\cdot)$ are all differentiable and bounded. $|a_1(\cdot)|\geq h_a>0$ and $|b(\cdot)|\geq h_b>0$. The equation $\boldsymbol{c}=\boldsymbol{A}(\boldsymbol{c})\boldsymbol{c}$ has a unique solution $\boldsymbol{c}=\boldsymbol{0}$. We learn the representation and dynamic model by solving (\ref{minimizing3}). Then $\boldsymbol{\hat z}=\tau(\boldsymbol{z})=\Delta b\cdot \boldsymbol{z}$
 	, where $\Delta b \in \mathbb{R}$ is a nonzero constant, $\boldsymbol{\hat A}\left(\boldsymbol{\hat z} \right) \boldsymbol{\hat z} = \Delta b \boldsymbol{A}(\boldsymbol{z})\boldsymbol{z}$ and $\hat b(\boldsymbol{\hat z})=\Delta b \cdot b(\boldsymbol{z})$.
 	\label{Theorem affine nonlinear single-input}
 \end{theorem}

  \begin{proof}
  	 In the affine nonlinear case, it is hard to directly deduce that $\tau(\cdot)$ is an affine function as what we did in the linear case. Hence, we instead first study the structure of the Jacobian of $\tau(\cdot)$, denoted as $J(\cdot)$. The proof mainly contains four steps:
  	 \begin{itemize}
  	 	\item In step 1, we follow \cite{ahuja2022weakly} to derive (\ref{nonlinear u}).
  	 	\item In step 2, we take the gradient of (\ref{nonlinear u}) w.r.t $u^{t}$, and leverage the structures of $\boldsymbol{B}(\cdot)$ and $\boldsymbol{\hat B}(\cdot)$ to show that the last column of $\boldsymbol{J}(\cdot)$ is $[0, \cdots, 0, \ast]^T$. Then by following similar procedures up to time step $t+n$ and taking the structures of $\boldsymbol{A}(\cdot)$ and $\boldsymbol{\hat A}(\cdot)$ in to consideration, we show that $\boldsymbol{J}(\cdot)$ is a lower triangular matrix (equation (\ref{J_:,:})).
  	 	\item In step 3, we let $u^t=0$ in (\ref{nonlinear u}) and take the gradient of (\ref{nonlinear u}) w.r.t $\boldsymbol{z}^{t}$. Then we use the structures of $\boldsymbol{A}(\cdot)$, $\boldsymbol{\hat A}(\cdot)$, and $\boldsymbol{J}(\cdot)$ to show that $\boldsymbol{J}(\cdot)$ is a diagonal matrix (equation (\ref{Jz diag})).
  	 	\item In step 4, we take the gradient of (\ref{nonlinear u}) w.r.t $\boldsymbol{z}^{t}$. After some calculation we can obtain $\frac{\hat b(\tau(\boldsymbol{z}^{t}))}{b(\boldsymbol{z}^{t})} \frac{\partial b(\boldsymbol{z}^{t})}{\partial \boldsymbol{z}^{t}} =\frac{\partial \hat b(\tau(\boldsymbol{z}^{t}))}{\partial \boldsymbol{z}^{t}}$, which gives $\frac{\hat b(\tau(\boldsymbol{z}^{t}))}{b(\boldsymbol{z}^{t})}= \Delta b$. Then we can prove all the conclusions in this theorem. 
  	 \end{itemize}

  	\textbf{Step 1.}
  	By following the proof technique in \cite{ahuja2022weakly}, solving (\ref{minimizing3}) implies 
  	\begin{align}
  		f(\boldsymbol{x}^{t+1})&=\boldsymbol{\hat A}\left( f(\boldsymbol{x}^{t}) \right) f(\boldsymbol{x}^{t})+\boldsymbol{\hat B}\left(f(\boldsymbol{x}^{t})\right) u ^{t} \notag \\
  		f \circ g(\boldsymbol{z}^{t+1})&=\boldsymbol{\hat A} \left( f\circ g(\boldsymbol{z}^{t}) \right) f\circ g(\boldsymbol{z}^{t})+\boldsymbol{\hat B}\left( f\circ g(\boldsymbol{z}^{t}) \right) u^{t} \notag \\
  		\tau(\boldsymbol{z}^{t+1})&=\boldsymbol{\hat A}\left(  \tau(\boldsymbol{z}^{t}) \right) \tau(\boldsymbol{z}^{t})+\boldsymbol{\hat B}\left( \tau(\boldsymbol{z}^{t})\right) u ^{t}.
  		\label{nonlinear estimated dynamic equation}
  	\end{align}
  	Substituting (\ref{nonlinear system with single-input}) into (\ref{nonlinear estimated dynamic equation}) gives
  	\begin{equation} 
  		\tau\left(\boldsymbol{A} (\boldsymbol{z}^{t})\boldsymbol{z}^{t} + \boldsymbol{B}(\boldsymbol{z}^{t}) u^{t}\right)=\boldsymbol{\hat A}\left(  \tau(\boldsymbol{z}^{t}) \right) \tau(\boldsymbol{z}^{t})+\boldsymbol{\hat B}\left(\tau(\boldsymbol{z}^{t})\right) u ^{t}.
  		\label{nonlinear u}
  	\end{equation}
  	
  	\textbf{Step 2.} Taking the gradient of (\ref{nonlinear u}) w.r.t $u^{t}$ gives 
  	\begin{align}
  		\boldsymbol{J}\left(\boldsymbol{A}(\boldsymbol{z}^{t})\boldsymbol{z}^{t} + \boldsymbol{B}(\boldsymbol{z}^{t}) u^{t}\right)\boldsymbol{B}(\boldsymbol{z}^{t})  =\boldsymbol{\hat B}\left(\tau(\boldsymbol{z}^{t})\right).
  		\label{nonlinear u J }
  	\end{align}
  	Let 
  	\begin{equation} 
  		\boldsymbol{y}(\boldsymbol{z}^{t},u^{t})=\boldsymbol{A}(\boldsymbol{z}^{t})\boldsymbol{z}^{t} + \boldsymbol{B}(\boldsymbol{z}^{t})u^{t}.
  		\label{y = Az}
  	\end{equation}
  	Equation (\ref{y = Az}) can be rewrite as
  	\begin{equation}
  		\left\{
  		\begin{array}{l}
  			y_1=z_2^t	\notag \\
  			y_2=z_3^t	\notag \\
  			\cdots \notag\\
  			y_{n-1}=z_n^t \notag \\
  			y_n=a_1(z^t)z_1^t+a_2(z^t)z_2^t+\cdots+a_n(z^t)z_n^t+b(z^t)u^t,
  		\end{array}
  		\right.
  	\end{equation} 
  	where $\forall i =1, \cdots, n$, $z_i^t$ and $y_i$ denote the $i$-th component of $\boldsymbol{z}^t$ and $\boldsymbol{y}$, respectively. If we fix $\boldsymbol{y}$ and $u$, above equation becomes
  	\begin{equation} 
  		z_1^t=\frac{y_n-a_2(z_1^t,y)y_1-\cdots-a_n(z_1^t,y)y_{n-1}-b(z_1^t,y)u^t}{a_1(z_1^t,y)}. 
  	\end{equation}
  	Since we assume $|a_1(\cdot)| \geq h_a >0$ and all the coefficient functions are differentiable and bounded. The right-hand side of above equation is a continuous and bounded function of $z_1^t$, which implies the existence of $z_1^t$, and hence implies the existence of $\boldsymbol{z}^t$ in equation (\ref{y = Az}).
  	
  	This gives that $\forall \boldsymbol{y}\in \mathbb{R}^n, u^t\in \mathbb{R}$, $\exists \boldsymbol{z}^t$ such that equation (\ref{y = Az}) holds.  Hence, when $\boldsymbol{z}^{t}$ varies within $\mathbb{R}^n$, $\boldsymbol{y}(\boldsymbol{z}^{t},u^{t})$ also varies within $\mathbb{R}^n$. Then we obtain that for all $\boldsymbol{y} \in  \mathbb{R}^n$,
  	\begin{align}
  		\boldsymbol{J}\left(\boldsymbol{y}(\boldsymbol{z}^{t},u^{t})\right)\boldsymbol{B}(\boldsymbol{z}^{t})  =\boldsymbol{\hat B}\left(\tau(\boldsymbol{z}^{t})\right),
  	\end{align}
  	and hence the last column of $\boldsymbol{J}(\boldsymbol{y}(\boldsymbol{z}^{t},u^{t}))$ has the following form
  	\begin{equation} 
  		\boldsymbol{J}_{\cdot,n}\left(\boldsymbol{y}(\boldsymbol{z}^{t},u^{t})\right) = [0,\cdots,0,\frac{\hat b\left(\tau(\boldsymbol{z}^{t})\right)}{b(\boldsymbol{z}^{t})}]^T.
  		\label{J_:,n}
  	\end{equation}

  	We then consider time step $t+2$ and let $u^{t+1}=0$. From (\ref{nonlinear system with single-input}) and (\ref{nonlinear estimated dynamic equation}), we obtain
  	\begin{equation}
  		\left\{
  		\begin{array}{l}
  			\boldsymbol{z}^{t+2}=\boldsymbol{A}(\boldsymbol{z}^{t+1})\boldsymbol{A}(\boldsymbol{z}^{t})\boldsymbol{z}^{t}+\boldsymbol{A}(\boldsymbol{z}^{t+1})\boldsymbol{B}(\boldsymbol{z}^{t})u^{t}  \notag \\ 
  			\tau(\boldsymbol{z}^{t+2})=\boldsymbol{\hat A}\left(  \tau(\boldsymbol{z}^{t+1}) \right) \boldsymbol{\hat A}\left(  \tau(\boldsymbol{z}^{t}) \right) \tau(\boldsymbol{z}^{t}) + \boldsymbol{\hat A}\left(  \tau(\boldsymbol{z}^{t+1}) \right) \boldsymbol{\hat B}(\tau(\boldsymbol{z}^{t})) u ^{t}     \notag,
  		\end{array}
  		\right.
  	\end{equation} 
  	which gives
  	\begin{align}
  		\tau\left(\boldsymbol{A}(\boldsymbol{z}^{t+1})\boldsymbol{A}(\boldsymbol{z}^{t})\boldsymbol{z}^{t}+\boldsymbol{A}(\boldsymbol{z}^{t+1})\boldsymbol{B}(\boldsymbol{z}^{t})u^{t}\right)  
  		= \boldsymbol{\hat A}\left(  \tau(\boldsymbol{z}^{t+1}) \right) \boldsymbol{\hat A}\left(  \tau(\boldsymbol{z}^{t}) \right) \tau(\boldsymbol{z}^{t}) + \boldsymbol{\hat A}\left(  \tau(\boldsymbol{z}^{t+1}) \right) \boldsymbol{\hat B}\left(\tau(\boldsymbol{z}^{t})\right) u ^{t}.
  		\label{nonlinear a u t2}
  	\end{align}
  	Let $\boldsymbol{y}'(\boldsymbol{z}^{t},u^{t})=\boldsymbol{A}(\boldsymbol{z}^{t+1})\boldsymbol{A}(\boldsymbol{z}^{t})\boldsymbol{z}^{t}+\boldsymbol{A}(\boldsymbol{z}^{t+1})\boldsymbol{B}(\boldsymbol{z}^{t})u^{t}$. Notice that $\boldsymbol{z}^{t+1}$ is a function of $\boldsymbol{z}^{t}$ and $u^{t}$, but according to the structure of $\boldsymbol{A}(\cdot)$, only items in the last row of $\boldsymbol{A}(\boldsymbol{z}^{t+1})$ are functions of $\boldsymbol{z}^{t+1}$ (and hence are functions of $\boldsymbol{z}^{t}$ and $u^{t}$). Taking the gradient of (\ref{nonlinear a u t2}) w.r.t $u^{t}$ gives 
  	\begin{align}
  		\boldsymbol{J}\left(\boldsymbol{y}'(\boldsymbol{z}^{t},u^{t})\right)\cdot [0,\cdots,0,b(\boldsymbol{z}^{t}),\ast]^T  
  		=  [0,\cdots,0,\hat b\left(\tau(\boldsymbol{z}^{t})\right),\ast]^T.
  		\label{nonlinear J u t2}
  	\end{align}
  	Again, the assumption $|a_1(\cdot)| \geq h_a >0$, all coefficient functions are differentiable and bounded, and the structure of $\boldsymbol{A}(\cdot)$ ensure that when $\boldsymbol{z}^{t}$ varies within $\mathbb{R}^n$, $\boldsymbol{y}'(\boldsymbol{z}^{t},u^{t})$ also varies within $\mathbb{R}^n$. Then (\ref{nonlinear J u t2}) implies $\forall \boldsymbol{y}' \in  \mathbb{R}^n$, the $(n-1)$-th column of $\boldsymbol{J}(\boldsymbol{y}'(\boldsymbol{z}^{t},u^{t}))$ has the following form
  	\begin{equation} 
  		\boldsymbol{J}_{\cdot,n-1}\left(\boldsymbol{y}'(\boldsymbol{z}^{t},u^{t})\right) = [0,\cdots,0,\frac{\hat b\left(\tau(\boldsymbol{z}^{t})\right)}{b(\boldsymbol{z}^{t})},\ast]^T,
  		\label{J_:,n-1}
  	\end{equation}
  	By following similar procedures as above up to time step $t+n$, we obtain that $\forall \boldsymbol{z} \in \mathbb{R}^n$, there exist $\boldsymbol{p}_1$, $\boldsymbol{p}_2$, $\cdots$, $\boldsymbol{p}_n \in \mathbb{R}^n$, such that
  	\begin{equation} 
  		\boldsymbol{J}(\boldsymbol{z}) = \begin{bmatrix}\frac{\hat b(\tau(\boldsymbol{p}_1))}{b(\boldsymbol{p}_1)}&0&0&0&0\\
  			*&\frac{\hat b(\tau(\boldsymbol{p}_2))}{b(\boldsymbol{p}_2)}&0&\cdots&0\\
  			\vdots&\vdots&\vdots&\ddots&\vdots\\
  			*&\cdots&*&\frac{\hat b(\tau(\boldsymbol{p}_{n-1}))}{b(\boldsymbol{p}_{n-1})}&0\\
  			*&*&\cdots&*&\frac{\hat b(\tau(\boldsymbol{p}_n))}{b(\boldsymbol{p}_n)}\end{bmatrix}.
  		\label{J_:,:}
  	\end{equation}

  	\textbf{Step 3.} Let $u^{t}=0$ in (\ref{nonlinear u}), we obtain
  	\begin{equation} 
  		\tau\left(\boldsymbol{A} (\boldsymbol{z}^{t})\boldsymbol{z}^{t}\right)=\boldsymbol{\hat A}\left(  \tau(\boldsymbol{z}^{t}) \right) \tau(\boldsymbol{z}^{t}).
  		\label{nonlinear u=0}
  	\end{equation}
  	
  	Let $\boldsymbol{\omega}\doteq \boldsymbol{y}(\boldsymbol{z}^t,0) = \boldsymbol{A}(\boldsymbol{z}^{t}) \boldsymbol{z}^{t}$, and take the gradient of (\ref{nonlinear u=0}) w.r.t $\boldsymbol{z}^t$, we obtain
  	\begin{equation}
  		\boldsymbol{J}(\boldsymbol{\omega}) \cdot \frac{\partial \boldsymbol{\omega}}{\partial \boldsymbol{z}^t} =\frac{\partial}{\partial \boldsymbol{z}^t}\left(\boldsymbol{\hat A}\left(  \tau(\boldsymbol{z}^{t}) \right) \tau(\boldsymbol{z}^{t})\right ), 
  		\label{omega d z}
  	\end{equation}
  	which gives
  	\begin{equation}
  		\boldsymbol{J}(\boldsymbol{\omega}) \cdot \begin{bmatrix}0&1&0&\cdots&0\\
  			0&0&1&\cdots&0\\
  			\vdots&\vdots&\vdots&\ddots&\vdots\\
  			0&0&0&\cdots&1\\
  			*&*&*&...&*\end{bmatrix} = \begin{bmatrix}J_{2,1}(\boldsymbol{z}^{t})&\cdots&J_{2,n}(\boldsymbol{z}^{t}) \\
  			\vdots&\ddots&\vdots\\
  			J_{n,1}(\boldsymbol{z}^{t})&\cdots&J_{n,n}(\boldsymbol{z}^{t})\\
  			*&\cdots&*\\ \end{bmatrix}.
  		\label{JA=J}
  	\end{equation}
  	Taking the structure of $\boldsymbol{J}(\cdot)$ into consideration (see (\ref{J_:,:})), we rewrite equation (\ref{JA=J}) as
  	\begin{align}
  		&\begin{bmatrix}*&0&\cdots&\cdots&0\\
  			J_{2,1}(\boldsymbol{\omega})&*&0&\cdots&0\\
  			\vdots&\vdots&\ddots&\vdots&\vdots\\
  			J_{n-1,1}(\boldsymbol{\omega})&\cdots&J_{n-2,n-1}(\boldsymbol{\omega})&*&0\\
  			J_{n,1}(\boldsymbol{\omega})&\cdots&\cdots&J_{n,n-1}(\boldsymbol{\omega})&*\end{bmatrix} 
  		\cdot 
  		\begin{bmatrix}0&1&0&\cdots&0\\
  			0&0&1&\cdots&0\\
  			\vdots&\vdots&\vdots&\ddots&\vdots\\
  			0&0&0&\cdots&1\\
  			*&*&*&...&*\end{bmatrix} \notag\\
  		=& 
  		\begin{bmatrix}J_{2,1}(z^{t})&*&0&\cdots&0\\
  			\vdots&\vdots&\ddots&\vdots&\vdots\\
  			J_{n-1,1}(z^{t})&\cdots&J_{n-2,n-1}(z^{t})&*&0\\
  			J_{n,1}(z^{t})&\cdots&\cdots&J_{n,n-1}(z^{t})&*\\
  			*&\cdots&\cdots&\cdots&*\\ \end{bmatrix}. 
  		\label{JA=J ex}
  	\end{align}

  	Considering the first column of both side of above equation, we obtain 
  	\begin{equation}
  		J_{2,1}(\boldsymbol{z}^{t})=J_{3,1}(\boldsymbol{z}^{t})=\cdots=J_{n,1}(\boldsymbol{z}^{t})=0.
  		\label{J:,1=0 z}
  	\end{equation}
  	Since above results hold for all $\boldsymbol{z}^t \in \mathbb{R}^n$, we also have
  	\begin{equation}
  		J_{2,1}(\boldsymbol{\omega})=J_{3,1}(\boldsymbol{\omega})=\cdots=J_{n,1}(\boldsymbol{\omega})=0.
  		\label{J:,1=0 Az}
  	\end{equation}
  	Substituting (\ref{J:,1=0 z}) and (\ref{J:,1=0 Az}) into (\ref{JA=J ex}) gives
  	\begin{align}
  		&\begin{bmatrix}*&0&\cdots&\cdots&\cdots&0\\
  			0&*&0&\cdots&\cdots&0\\
  			0&J_{3,2}(\boldsymbol{\omega}) &*&\cdots&\cdots&0\\
  			\vdots&\vdots&\vdots&\ddots&\vdots&\vdots\\
  			0&J_{n-1,2}(\boldsymbol{\omega})&\cdots&J_{n-2,n-1}(\boldsymbol{\omega})&*&0\\
  			0&J_{n,2}(\boldsymbol{\omega})&\cdots&\cdots&J_{n,n-1}(\boldsymbol{\omega})&*\end{bmatrix} 
  		\cdot 
  		\begin{bmatrix}0&1&0&\cdots&0\\
  			0&0&1&\cdots&0\\
  			\vdots&\vdots&\vdots&\ddots&\vdots\\
  			0&0&0&\cdots&1\\
  			*&*&*&...&*\end{bmatrix} \notag\\
  		=& 
  		\begin{bmatrix}0&*&0&\cdots&\cdots&0\\
  			0&J_{3,2}(\boldsymbol{z}^{t})&*&0&\cdots&0\\
  			\vdots&\vdots&\vdots&\ddots&\vdots&\vdots\\
  			0&J_{n-1,2}(\boldsymbol{z}^{t})&\cdots&J_{n-2,n-1}(\boldsymbol{z}^{t})&*&0\\
  			0&J_{n,2}(\boldsymbol{z}^{t})&\cdots&\cdots&J_{n,n-1}(\boldsymbol{z}^{t})&*\\
  			*&\cdots&\cdots&\cdots&*\\ \end{bmatrix}. 
  	\end{align}
  	Considering the second column of both side of above equation, we obtain 
  	\begin{equation}
  		J_{3,2}(\boldsymbol{z}^{t})=J_{4,2}(\boldsymbol{z}^{t})=\cdots=J_{n,2}(\boldsymbol{z}^{t})=0,
  	\end{equation}
  	and
  	\begin{equation}
  		J_{3,2}(\boldsymbol{\omega})=J_{4,2}(\boldsymbol{\omega})=\cdots=J_{n,2}(\boldsymbol{\omega})=0.
  	\end{equation}
  	By repeating above procedures, we obtain that $\boldsymbol{J}(\cdot)$ is a diagonal matrix, hence we can rewrite (\ref{J_:,:}) as 
  	\begin{equation}
  		\boldsymbol{J}(\boldsymbol{z})=\mbox{diag} \{\frac{\hat b(\tau(\boldsymbol{p}_1))}{b(\boldsymbol{p}_1)}, \cdots, \frac{\hat b(\tau(\boldsymbol{p}_n))}{b(\boldsymbol{p}_n)}\}.
  		\label{Jz diag}
  	\end{equation}

  	\textbf{Step 4.} Taking the gradient of (\ref{nonlinear u}) w.r.t $\boldsymbol{z}^t$, we obtain
  	\begin{equation}
  		\boldsymbol{J}\left(\boldsymbol{y}(\boldsymbol{z}^{t},u^{t})\right)\left[\frac{\partial \boldsymbol{A} (\boldsymbol{z}^{t})\boldsymbol{z}^{t}}{\partial \boldsymbol{z}^t} + \frac{\partial \boldsymbol{B}(\boldsymbol{z}^{t})}{\partial \boldsymbol{z}^t} u^{t}\right] = \frac{\partial \boldsymbol{\hat A}\left(  \tau(\boldsymbol{z}^{t}) \right) \tau(\boldsymbol{z}^{t})}{\partial \boldsymbol{z}^t} + \frac{\partial \boldsymbol{\hat B}\left(\tau(\boldsymbol{z}^{t})\right)}{\partial \boldsymbol{z}^t} u ^{t}.
  	\end{equation}
  	We use $[\boldsymbol{M}]_{n,\cdot}$ to denote the $n$-th row of a matrix $\boldsymbol{M}$, and write the last line of above equation as
  	\begin{equation}
  		\boldsymbol{J}_{n,\cdot}\left(\boldsymbol{y}(\boldsymbol{z}^{t},u^{t})\right)\left[\frac{\partial \boldsymbol{A} (\boldsymbol{z}^{t})\boldsymbol{z}^{t}}{\partial \boldsymbol{z}^t} + \frac{\partial \boldsymbol{B}(\boldsymbol{z}^{t})}{\partial \boldsymbol{z}^t} u^{t}\right] = \left[\frac{\partial \boldsymbol{\hat A}\left(  \tau(z^{t}) \right) \tau(\boldsymbol{z}^{t})}{\partial \boldsymbol{z}^t}\right]_{n,\cdot} + \left[\frac{\partial \boldsymbol{\hat B}\left(\tau(\boldsymbol{z}^{t})\right)}{\partial \boldsymbol{z}^t} u ^{t}\right]_{n,\cdot}
  		\label{last line y}.
  	\end{equation}
  	We then write the last line of equation (\ref{omega d z}) as 
  	\begin{equation}
  		\boldsymbol{J}_{n,\cdot}(\boldsymbol{\omega})\frac{\partial \boldsymbol{\omega}}{\partial \boldsymbol{z}^t} = \left[\frac{\partial \boldsymbol{\hat A}\left(  \tau(\boldsymbol{z}^{t}) \right) \tau(\boldsymbol{z}^{t})}{\partial \boldsymbol{z}^t}\right]_{n,\cdot} 
  		\label{last line omega}
  	\end{equation}
  	Since $\boldsymbol{J}(\cdot)$ is a diagonal matrix (see (\ref{Jz diag})), from (\ref{J_:,n}) we obtain that
  	\begin{equation} 
  		\boldsymbol{J}_{n,\cdot}\left(\boldsymbol{y}(\boldsymbol{z}^{t},u^{t})\right) = \left[0,\cdots,0,\frac{\hat b(\tau(\boldsymbol{z}^{t}))}{b(\boldsymbol{z}^{t})}\right] =\boldsymbol{J}_{n,\cdot}\left(\boldsymbol{y}(\boldsymbol{z}^{t},0)\right) = \boldsymbol{J}_{n,\cdot}(\boldsymbol{\omega})
  		\label{J_{n,:}}
  	\end{equation} 
  	Then from (\ref{last line y}), (\ref{last line omega}), and (\ref{J_{n,:}}) we obtain that
  	\begin{equation}
  		\left[0,\cdots,0,\frac{\hat b(\tau(\boldsymbol{z}^{t}))}{b(\boldsymbol{z}^{t})}\right] \cdot \left[ \frac{\partial \boldsymbol{B}(\boldsymbol{z}^{t})}{\partial \boldsymbol{z}^{t}} u^{t}\right] = \left[\frac{\partial \boldsymbol{\hat B}\left(\tau(\boldsymbol{z}^{t})\right)}{\partial \boldsymbol{z}^{t}} u ^{t}\right]_{n,\cdot},
  		\label{last line}
  	\end{equation}
  	which gives
  	\begin{equation}
  		\frac{\hat b(\tau(\boldsymbol{z}^{t}))}{b(\boldsymbol{z}^{t})} \cdot \frac{\partial b(\boldsymbol{z}^{t})}{\partial \boldsymbol{z}^{t}} =\frac{\partial \hat b(\tau(\boldsymbol{z}^{t}))}{\partial \boldsymbol{z}^{t}}.
  	\end{equation}
  	Above equation means
  	\begin{equation}
  		\frac{\partial \left[\hat b(\tau(\boldsymbol{z}^{t}))/ b(\boldsymbol{z}^{t})\right]}{\partial \boldsymbol{z}^{t}} =0,
  	\end{equation}
  	and hence
  	\begin{equation}
  		\frac{\hat b(\tau(\boldsymbol{z}^{t}))}{b(\boldsymbol{z}^{t})} = \Delta b,
  	\end{equation}
  	where $\Delta b$ is a nonzero constant. Then from (\ref{Jz diag}), we have 
  	\begin{equation}
  		\boldsymbol{J}(\boldsymbol{z})= \Delta b \boldsymbol{I}_n,
  	\end{equation}
  	and hence
  	\begin{equation}
  		\tau(\boldsymbol{z})=\Delta b \cdot \boldsymbol{z} + \boldsymbol{c}.
  		\label{tau z =c}
  	\end{equation}
  	Substituting equation (\ref{tau z =c}) into (\ref{nonlinear u=0}) gives
  	\begin{equation} 
  		\Delta b \cdot \boldsymbol{A}(\boldsymbol{z}^{t}) \boldsymbol{z}^{t} + \boldsymbol{c}=\boldsymbol{\hat A}\left(\Delta b \cdot \boldsymbol{z}^{t}+\boldsymbol{c} \right) (\Delta b \cdot \boldsymbol{z}^{t} + \boldsymbol{c}).
  		\label{hat A=A}
  	\end{equation}
  	Let $\boldsymbol{z}^{t}=-\dfrac{1}{\Delta b}\cdot \boldsymbol{c}$, we obtain
  	\begin{align} 
  		-\dfrac{1}{\Delta b}\cdot \boldsymbol{c}= \boldsymbol{A}\left( -\dfrac{1}{\Delta b}\cdot \boldsymbol{c} \right)\left(-\dfrac{1}{\Delta b}\cdot \boldsymbol{c}\right).
  		\label{c=hat A c}
  	\end{align} 
  	By assumption, above equation implies $\dfrac{1}{\Delta b}\cdot \boldsymbol{c}=\boldsymbol{0}$, and hence $\boldsymbol{c}=\boldsymbol{0}$. Then we have $\forall \boldsymbol{z} \in \mathbb{R}^n$, $\tau(\boldsymbol{z})=\Delta b \cdot\boldsymbol{z}$ and $\hat b(\Delta b \cdot \boldsymbol{z}) =\Delta b\cdot b(\boldsymbol{z})$.
  	Substituting $\boldsymbol{c}=\boldsymbol{0}$ into (\ref{hat A=A}) gives
  	\begin{equation} 
  		\boldsymbol{A}(\boldsymbol{z}) \boldsymbol{z} =\boldsymbol{\hat A}\left( \Delta b \cdot \boldsymbol{z} \right) \boldsymbol{z}.
  	\end{equation}

  \end{proof}
  
  \begin{remark}
  	  {Most assumptions in Theorem \ref{Theorem affine nonlinear single-input} are similar to Theorem \ref{Theorem linear single-input}, except for the differentiability and $|a_1(\cdot)|\geq h_a>0$. The differentiability assumption is necessary, and it holds for many physical systems like spacecrafts and robot arms, as long as they are not switching systems. The assumption $|a_1(\cdot)|\geq h_a>0$
  		is actually not necessary for our theoretical results.
  	
  	The reason why we assume $|a_1(\cdot)|>0$ is to ensure that all the states in $\mathbb{R}^n$ are reachable, i.e., $\forall \boldsymbol{z}^{t+1} \in \mathbb{R}^n$, $\exists \boldsymbol{z}^{t}$ and $u$ that make the dynamic system equation hold. Then our results will hold for $\mathbb{R}^n$. If the control input can be set to an arbitrarily large value, then we do not need this assumption, which is the same as we discussed in the linear case. In practice, how large the input is required depends on the system and the actual range of the state $z^{t}$. (Notice that $a_1(z^{t})=0$ does not mean we need an infinite input at this state. An input that can compensate for all other terms is enough.)
  	
  	On the other hand, even if $a_1(z^t)=0$ and the input fails to make all the states in $\mathbb{R}^n$ are reachable, it is still not a very big deal. In this case, our results will hold for the reachable set, which is a subset of $\mathbb{R}^n$. Unless we deliberately set the initial state to points in the unreachable set, the system will never arrive at the set that the identifiable results do not hold. Hence, this assumption is not necessary in practice.}
  \end{remark}

  \begin{remark} 
  	
  	Theorem \ref{Theorem affine nonlinear single-input} shows that, even if the system is affine nonlinear, we can identify the latents up to scaling, and the dynamic model is identified up to a trivial transformation, which will not affect the prediction. This is because the one-step prediction from the model
 \begin{align} 
 	\boldsymbol{\tilde z}^{t+1}\doteq& \boldsymbol{\hat A}(\boldsymbol{\hat z}^t)\boldsymbol{\hat z}^t +\hat b(\boldsymbol{\hat z}^t)u^t \label{tilde z^{t+1}} \\
 	=&\Delta b\cdot \boldsymbol{A}(\boldsymbol{z}^t)\boldsymbol{z}^t+\Delta b\cdot b(\boldsymbol{z}^t)u\notag \\
 	=&\Delta b \cdot \boldsymbol{z}^{t+1} \notag \\
 	=&\boldsymbol{\hat z}^{t+1}\notag . 	
 \end{align}
 
 { If we want to identify the dynamic model one step further, e.g., if we  require that all the coefficient functions are identifiable, additional assumptions have to be made. One typical assumption is to let the coefficient function $a_i(\cdot)$ only depend on $q^{t},q^{t+1},\cdots,q^{t+i-1}$, i.e., the dynamic model (\ref{nonlinear system with single-input plain}) is 
 modified to
  \begin{equation}
 	q^{t+n}-a_{n}(q^{t},q^{t+1},\cdots,q^{t+n-1})q^{t+n-1}-\cdots-a_{1}(q^{t})q^{t} = b(\boldsymbol{z}^{t}) u^{t}.
 \end{equation}
 And the estimated coefficient functions are modified accordingly. Then all the coefficient functions are identifiable, which can be easily derived by using the conclusion of Theorem \ref{Theorem affine nonlinear single-input}.}
\end{remark}

 \subsection{Extensions to Multi-Input Dynamic Systems {with sparse input matrices}}
 In this section, we discuss extending the results in Section \ref{Section Single-Input Latent Dynamic Systems} to multi-input dynamic systems {with sparse input matrices}. If, in the controllable canonical form, the latent variables that each input can affect have no intersections, the conclusions for single-input systems can be easily extended to the multi-input cases. Otherwise, we use the Luenberger controllable canonical form \cite{luenberger1967canonical}, which can model the interactive effect of each subsystem. We only discuss the latter case.
 
 \subsubsection{Linear Cases}
 For linear multi-input systems
 \begin{equation}
 	\boldsymbol{z}^{t+1}=\boldsymbol{\bar A} \boldsymbol{z}^{t} + \boldsymbol{\bar B} \boldsymbol{u}^{t},
 	\label{linear system with multi-input}
 \end{equation}
 where $\boldsymbol{z}^{t} \in  \mathbb{R}^{nd}$, $\boldsymbol{u}^{t} \in \mathbb{R}^d$. The set of the first $n$ equations in (\ref{linear system with multi-input}) is referred to as the first subsystem, and so on. Coefficient matrices $\boldsymbol{\bar A}\in \mathbb{R}^{nd \times nd}$ and $\boldsymbol{\bar B} \in \mathbb{R}^{nd \times d}$ in (\ref{linear system with multi-input}) are given as
 \begin{equation}
 	\boldsymbol{\bar A}=\begin{bmatrix}\boldsymbol{A}^{11}&\cdots&\boldsymbol{A}^{1d}\\ 
 		\vdots&\ddots&\vdots\\ 
 		\boldsymbol{A}^{d1}&\cdots&\boldsymbol{A}^{dd}\end{bmatrix}, \boldsymbol{\bar B}=\begin{bmatrix}\boldsymbol{B}^1\\\vdots\\ \boldsymbol{B}^d\end{bmatrix},
 	\label{linear system with multi-input A B}
 \end{equation}
 where $\forall i,j=1,\cdots,d$, $\boldsymbol{A}^{ii}\in \mathbb{R}^{n \times n}$, $\boldsymbol{A}^{ij}\in \mathbb{R}^{n \times n}$, and $\boldsymbol{B}^{i}\in \mathbb{R}^{n \times d}$ have following forms
 
 \begin{align}
 	\boldsymbol{A}^{ii}=\begin{bmatrix}0&1&0&\cdots&0\\
 		0&0&1&\cdots&0\\
 		\vdots&\vdots&\vdots&\ddots&\vdots\\
 		0&0&0&\cdots&1\\
 		a^{ii}_1& a^{ii}_2& a^{ii}_3&...& a^{ii}_n\end{bmatrix},
 	\boldsymbol{A}^{ij}=\begin{bmatrix}0&0&0&\cdots&0\\
 		0&0&0&\cdots&0\\
 		\vdots&\vdots&\vdots&\ddots&\vdots\\
 		0&0&0&\cdots&0\\
 		a^{ij}_1& a^{ij}_2& a^{ij}_3&...& a^{ij}_n\end{bmatrix}, 
 	\label{linear system with multi-input Aij}
 \end{align}
 \begin{align}
 	\boldsymbol{B}^i=\begin{bmatrix}0&\cdots&\cdots&\cdots&\cdots&\cdots&0\\
 		\vdots&\ &\ &\vdots&\ &\ &\vdots\\
 		0&\cdots&0&b^i&0&\cdots&0\end{bmatrix}. 
 	\label{linear system with multi-input Bi}
 \end{align}
Only the $(n,i)$-th element in $\boldsymbol{B}^i$ is nonzero, which means $\boldsymbol{B}^i$ and the input matix $\boldsymbol{\bar B}$ are sparse. 

    The mixing function $g: \mathbb{R}^{nd} \to \mathcal{X} \subset \mathbb{R}^m$, $m\geq nd$, is similarly defined as in (\ref{g}), except for the dimension of the input.
 
 The estimated coefficient matrices $\boldsymbol{\hat {\bar A}}$ and $\boldsymbol{\hat {\bar B}}$ have the same structures as in (\ref{linear system with multi-input A B}), (\ref{linear system with multi-input Aij}), and (\ref{linear system with multi-input Bi}). $\forall i=1, \cdots, d$, $\hat b^i$ is set to arbitrary nonzero constant in advance, which determines $\boldsymbol{\hat {\bar B}}$. The representation function $f$: $\mathcal{X} \to \mathbb{R}^{nd}$ and the estimated coefficient matrix $\boldsymbol{\hat {\bar A}}$ are learned by solving 
  
 \begin{equation}
 	\mathop{\min}_{\substack{\boldsymbol{\hat {\bar A}}, f(\cdot)}} \mathop{E}  [|| f(\boldsymbol{x}^{t+1})-\boldsymbol{\hat {\bar A}} f(\boldsymbol{x}^{t})-\boldsymbol{\hat {\bar B}} \boldsymbol{u}^t||^2].
 	\label{minimizing2}
 \end{equation} 
 
 Then we have the following identifiability result.
 \begin{theorem}
 	Suppose the latent dynamic system is represented as in (\ref{linear system with multi-input})  and (\ref{g}). The coefficients satisfy $\forall i =1, \cdots, d$, $a^{ii}_1 \neq 0$, $b^i \neq 0$, and $\boldsymbol{I}_{nd}-\boldsymbol{\bar A}$ is nonsingular. We learn the representation and dynamic model by solving (\ref{minimizing2}). Then $\tau(\boldsymbol{z})=\mbox{diag}\{\frac{\hat b^1}{b^1}\boldsymbol{I}_{n}, \cdots, \frac{\hat b^d}{b^d}\boldsymbol{I}_{n}\}\cdot \boldsymbol{z}$. $\forall i,j=1,\cdots,d$,  $\boldsymbol{\hat {A}}^{ii}=\boldsymbol{A}^{ii}$ and $\boldsymbol{\hat {A}}^{ij}=\frac{\hat b^i b^j}{b^i \hat b^j}\boldsymbol{A}^{ij}$.
 	\label{Theorem linear muti-input}
 \end{theorem}

 \begin{proof}
 	By following step 1 in the proof of Theorem \ref{Theorem linear single-input}, under the assumption that $\forall i =1, \cdots, d$, $a^{ii}_1 \neq 0$, we can easily
 	prove that $\tau(\cdot)$ is an affine map, i.e.,
 	\begin{equation}
 		\boldsymbol{\hat z} = \tau(\boldsymbol{z}) = \boldsymbol{T} \boldsymbol{z} + \boldsymbol{c},
 		\label{linear multi-input a}
 	\end{equation}
 	where $\boldsymbol{T} \in \mathbb{R}^{nd \times nd}, \boldsymbol{c} \in \mathbb{R}^{nd}$, and hence
 	\begin{equation}
 		\left\{
 		\begin{array}{l}
 			\boldsymbol{T}\boldsymbol{\bar A} = \boldsymbol{\hat {\bar A}} \boldsymbol{T} \notag\\
 			\boldsymbol{T}\boldsymbol{\bar A}^{i}\boldsymbol{\bar B} = \boldsymbol{\hat {\bar A}}^{i}\boldsymbol{\hat {\bar B}} , \forall i = 0, \cdots, n-1 \notag\\
 			\boldsymbol{c}=\boldsymbol{\hat {\bar A}} \boldsymbol{c}.
 		\end{array}
 		\right.
 	\end{equation}
 	
 	From $\boldsymbol{T}\boldsymbol{\bar A}^{i}\boldsymbol{\bar B} = \boldsymbol{\hat {\bar A}}^{i}\boldsymbol{\hat {\bar B}}$, after some calculation, we have 
 	\begin{equation}
 		\boldsymbol{T}^{ii}=\begin{bmatrix}\frac{\hat b^i}{b^i}&0&\cdots&0&0\\
 			*&\frac{\hat b^i}{b^i}&0&\cdots&0\\
 			\vdots&\vdots&\vdots&\ddots&\vdots\\
 			*&\cdots&*&\frac{\hat b^i}{b^i}&0\\
 			*&*&\cdots&*&\frac{\hat b^i}{b^i}\end{bmatrix},
 		\boldsymbol{T}^{ij}=\begin{bmatrix}0&0&\cdots&0&0\\
 			*&0&0&\cdots&0\\
 			\vdots&\vdots&\vdots&\ddots&\vdots\\
 			*&\cdots&*&0&0\\
 			*&*&\cdots&*&0\end{bmatrix}. 
 	\end{equation}
 	
 	Then from $\boldsymbol{T}\boldsymbol{\bar A} = \boldsymbol{\hat {\bar A}} \boldsymbol{T}$, we obtain $\boldsymbol{T}=\mbox{diag}\{\frac{\hat b^1}{b^1}\boldsymbol{I}_{n}, \cdots, \frac{\hat b^d}{b^d}\boldsymbol{I}_{n}\}$, $\boldsymbol{\hat {A}}^{ii}=\boldsymbol{A}^{ii}$, and $\boldsymbol{\hat {A}}^{ij}=\frac{\hat b^i b^j}{b^i \hat b^j}\boldsymbol{A}^{ij}$. 
 	
 	We rewrite $\boldsymbol{c}=\boldsymbol{\hat {\bar A}} \boldsymbol{c}$ as $\boldsymbol{T}^{-1}\boldsymbol{c}=\boldsymbol{T}^{-1}\boldsymbol{\hat {\bar A}} \boldsymbol{T} \boldsymbol{T}^{-1}\boldsymbol{c}$, which gives $\boldsymbol{T}^{-1}\boldsymbol{c}=\boldsymbol{\bar A} \boldsymbol{T}^{-1}\boldsymbol{c}$. Since $\boldsymbol{I}-\boldsymbol{\bar A}$ is nonsingular, $\boldsymbol{T}^{-1}\boldsymbol{c}=\boldsymbol{\bar A} \boldsymbol{T}^{-1}\boldsymbol{c}$ yields $\boldsymbol{T}^{-1}\boldsymbol{c}=\boldsymbol{0}$, and hence $\boldsymbol{c}=\boldsymbol{0}$. Then we have $\tau(\boldsymbol{z})=\mbox{diag}\{\frac{\hat b^1}{b^1}\boldsymbol{I}_{n}, \cdots, \frac{\hat b^d}{b^d}\boldsymbol{I}_{n}\}\cdot \boldsymbol{z}$.

 \end{proof}

  Theorem \ref{Theorem linear muti-input} shows that, for multi-input linear systems, we can still identify the latents up to scaling. All the coefficients in the diagonal blocks $\boldsymbol{A}^{ii}$ are identified up to the truth value. Other estimated coefficients are scaled due to the estimation error of $b^i$. However, it is easy to verify that this scaling will not affect one-step predictions from the learned latent model. Since if we let $\boldsymbol{T}=\mbox{diag}\{\frac{\hat b^1}{b^1}\boldsymbol{I}_{n}, \cdots, \frac{\hat b^d}{b^d}\boldsymbol{I}_{n}\}$, then $\boldsymbol{T}\boldsymbol{\bar A}=\boldsymbol{\hat {\bar A}} \boldsymbol{T}$ and $\boldsymbol{T} \boldsymbol{\bar B} = \boldsymbol{\hat {\bar B}}$. Hence we have
 \begin{align} 
 	\boldsymbol{\tilde z}^{t+1}\doteq& \boldsymbol{\hat {\bar A}} \boldsymbol{\hat z}^t +\boldsymbol{\hat {\bar B}}u^t \label{tilde z^{t+1} multi linear} \\
 	=&\boldsymbol{\hat {\bar A}} \boldsymbol{T} \boldsymbol{z}^t +\boldsymbol{T} \boldsymbol{\bar B} u^t\notag \\
 	=&\boldsymbol{T} (\boldsymbol{\bar A} \boldsymbol{z}^t + \boldsymbol{\bar B} u^t)  \notag \\
 	=&\boldsymbol{T}\boldsymbol{z}^{t+1} =\boldsymbol{\hat z}^{t+1}\notag . 	
 \end{align}
  \begin{remark}
 	{
 		There are two core assumptions in the dynamic system (\ref{linear system with multi-input}):
 		\begin{itemize}

 			\item 1. The input matrix $\boldsymbol{\bar B}$ is sparse. This is from the sparse perturbations assumption, which is wildly used in causal representation learning (see, e.g., \cite{ahuja2022weakly,ahuja2023interventional,lachapelle2024nonparametric,lippe2022citris}). If this assumption does not hold, it may be impossible to derive scaling identifiability without further assumptions.  
 			
 			\item 2. The orders are the same for all subsystems. If this assumption does not hold, we need some additional assumptions. A typical assumption is that the system is uncoupled, i.e., $\forall i \neq j$, $\boldsymbol{A}^{ij}=\boldsymbol{0}$. Then all the results still hold.  
 		\end{itemize}
 	}
 \end{remark}
 
 \subsubsection{Affine Nonlinear Cases}
 For affine nonlinear multi-input systems that can represented as
 \begin{equation}
 	\boldsymbol{z}^{t+1}=\boldsymbol{\bar A}(\boldsymbol{z}) \boldsymbol{z}^{t} + \boldsymbol{\bar B}(\boldsymbol{z}) \boldsymbol{u}^{t}.
 	\label{nonlinear system with multi-input}
 \end{equation}
 $\boldsymbol{\bar A}(\boldsymbol{z})\in \mathbb{R}^{nd \times nd}$ and $\boldsymbol{\bar B}(\boldsymbol{z}) \in \mathbb{R}^{nd \times d}$ are given as
 \begin{equation}
 	\boldsymbol{\bar A}(\boldsymbol{z})=\begin{bmatrix}\boldsymbol{A}^{11}(z)&\cdots&\boldsymbol{A}^{1d}(z)\\ 
 		\vdots&\ddots&\vdots\\ 
 		\boldsymbol{A}^{d1}(\boldsymbol{z})&\cdots&\boldsymbol{A}^{dd}(z)\end{bmatrix}, \boldsymbol{\bar B}(\boldsymbol{z})=\begin{bmatrix}\boldsymbol{B}^1(\boldsymbol{z})\\\vdots\\ \boldsymbol{B}^d(\boldsymbol{z})\end{bmatrix},
 	\label{nonlinear system with multi-input A B}
 \end{equation}
 where $\forall i,j=1,\cdots,d$, $\boldsymbol{A}^{ii}(z)\in \mathbb{R}^{n \times n}$, $\boldsymbol{A}^{ij}(z)\in \mathbb{R}^{n \times n}$, and $\boldsymbol{B}^{i}(z)\in \mathbb{R}^{n \times d}$ have following forms
 \begin{align}
 \boldsymbol{A}^{ii}(\boldsymbol{z})=\begin{bmatrix}0&1&0&\cdots&0\\
 		0&0&1&\cdots&0\\
 		\vdots&\vdots&\vdots&\ddots&\vdots\\
 		0&0&0&\cdots&1\\
 		a^{ii}_1(z)& a^{ii}_2(z)& a^{ii}_3(z)&...& a^{ii}_n(z)\end{bmatrix},
 	\boldsymbol{A}^{ij}(\boldsymbol{z})=\begin{bmatrix}0&0&0&\cdots&0\\
 		0&0&0&\cdots&0\\
 		\vdots&\vdots&\vdots&\ddots&\vdots\\
 		0&0&0&\cdots&0\\
 		a^{ij}_1(z)& a^{ij}_2(z)& a^{ij}_3(z)&...& a^{ij}_n(z)\end{bmatrix}, 
 	\label{nonlinear system with multi-input Aij}
 \end{align}
 \begin{equation}
 	\boldsymbol{B}^i(\boldsymbol{z})=\begin{bmatrix}0&\cdots&\cdots&\cdots&\cdots&\cdots&0\\
 		\vdots&\ &\ &\vdots&\ &\ &\vdots\\
 		0&\cdots&0&b^i(z)&0&\cdots&0\end{bmatrix}.
 	\label{nonlinear system with multi-input Bi}
 \end{equation}
 All the coefficient functions in $\boldsymbol{\bar A}(\boldsymbol{z})$ and $\boldsymbol{\bar B}(\boldsymbol{z})$ are assumed to be differentiable and bounded. The estimated coefficient function matrices $\boldsymbol{\hat {\bar A}}(\boldsymbol{z})$ and $\boldsymbol{\hat {\bar B}}(\boldsymbol{z})$ have the same structures as (\ref{nonlinear system with multi-input}), (\ref{nonlinear system with multi-input A B}), (\ref{nonlinear system with multi-input Aij}), and (\ref{nonlinear system with multi-input Bi}), and all the estimated coefficient functions in $\boldsymbol{\hat {\bar A}}(\boldsymbol{z})$ and $\boldsymbol{\hat {\bar B}}(\boldsymbol{z})$ are assumed to be differentiable. 
 
 The representation function $f$ and estimated coefficient function matrices $\boldsymbol{\hat {\bar A}}(\cdot)$ and  $\boldsymbol{\hat {\bar B}}(\cdot)$ are learned by solving 
 \begin{align}
 	\mathop{\min}_{\substack{\boldsymbol{\hat {\bar A}}(\cdot), \boldsymbol{\hat {\bar B}}(\cdot), f(\cdot)}}& \mathop{E} [|| f(\boldsymbol{x}^{t+1})-\boldsymbol{\hat {\bar A}}(f(\boldsymbol{x}^{t})) f(\boldsymbol{x}^{t})-\boldsymbol{\hat {\bar B}}(f(\boldsymbol{x}^{t})) \boldsymbol{u}^t||^2]\notag \\
 	\text{s.t.}\quad \ \ \ &|\hat b^i(\cdot)|\geq \hat h_b>0, \forall i=1, \cdots, d.
 	\label{minimizing4}
 \end{align}

 Then we have the following identifiability result.
 
 \begin{theorem}
 	Suppose the latent dynamic system is represented as in (\ref{nonlinear system with multi-input}) and (\ref{g}). All the coefficient functions are differentiable and bounded. $\forall i=1,\cdots,d$, $|a^{ii}_1(\cdot)|\geq h_a>0$, $|b^i(\cdot)|\geq h_b>0$. The equation $\boldsymbol{c}=\boldsymbol{\bar A}(\boldsymbol{c})\boldsymbol{c}$ has the unique solution $\boldsymbol{c}=\boldsymbol{0}$. We learn the representation and dynamic model by solving (\ref{minimizing4}). Then $\forall i =1, \cdots, d$, $\hat b^i(\Delta b^i \cdot \boldsymbol{z})=\Delta b^i \cdot b^i(\boldsymbol{z})$, where $\Delta b^i \in \mathbb{R}$ is a nonzero constant. Let $\Delta \boldsymbol{\bar{B}} =\mbox{diag}\{\Delta b^1 \boldsymbol{I}_{n}, \cdots, \Delta b^d \boldsymbol{I}_{n}\}$, then $  \boldsymbol{\hat z}=\tau(\boldsymbol{z})=\Delta \boldsymbol{\bar{B}}\cdot \boldsymbol{z}$ and $\boldsymbol{\hat {\bar A}}\left( \boldsymbol{\hat z} \right) \boldsymbol{\hat z} = \Delta \boldsymbol{\bar{B}} \boldsymbol{\bar A}(\boldsymbol{z}) \boldsymbol{z}$.
 	\label{Theorem affine nonlinear multi-input}
 \end{theorem}
 
\begin{proof}
	For the affine nonlinear multi-input systems, by using similar procedures as in the prove of Theorem \ref{Theorem affine nonlinear single-input}, we still have
	\begin{align}
		\boldsymbol{J}\left(\boldsymbol{\bar A}(\boldsymbol{z}^{t})\boldsymbol{z}^{t} + \boldsymbol{\bar B}(\boldsymbol{z}^{t}) \boldsymbol{u}^{t}\right)\boldsymbol{\bar B}(\boldsymbol{z}^{t})  =\boldsymbol{\hat {\bar B}}\left(\tau(\boldsymbol{z}^{t})\right),
		\label{nonlinear u J multi}
	\end{align}
	\begin{equation}
		\boldsymbol{J}(\boldsymbol{\omega}) \cdot \frac{\partial \boldsymbol{\omega}}{\partial \boldsymbol{z}^t} =\frac{\partial}{\partial \boldsymbol{z}^t}\left(\boldsymbol{\hat {\bar A}}\left(  \tau(\boldsymbol{z}^{t}) \right) \tau(\boldsymbol{z}^{t})\right ), 
		\label{omega d z multi}
	\end{equation}
	where $\boldsymbol{\omega}\doteq \boldsymbol{\bar A}(\boldsymbol{z}^{t}) \boldsymbol{z}^{t}$, and
	\begin{equation}
		\boldsymbol{J}\left(\boldsymbol{y}(\boldsymbol{z}^{t},\boldsymbol{u}^{t})\right)\left[\frac{\partial \boldsymbol{\bar A} (\boldsymbol{z}^{t})\boldsymbol{z}^{t}}{\partial \boldsymbol{z}^t} + \frac{\partial \boldsymbol{\bar B}(\boldsymbol{z}^{t})}{\partial \boldsymbol{z}^t} \boldsymbol{u}^{t}\right] = \frac{\partial \boldsymbol{\hat {\bar A}}\left(  \tau(\boldsymbol{z}^{t}) \right) \tau(\boldsymbol{z}^{t})}{\partial \boldsymbol{z}^t} + \frac{\partial \boldsymbol{\hat {\bar B}}\left(\tau(\boldsymbol{z}^{t})\right)}{\partial \boldsymbol{z}^t} \boldsymbol{u} ^{t}.
		\label{temp multi}
	\end{equation}
	
	By following similar procedures as in the proofs of Theorem \ref{Theorem affine nonlinear single-input} and Theorem \ref{Theorem linear muti-input}, we can easily show that the Jacobian $\boldsymbol{J}(\cdot)$ has the following structure:
	\begin{equation}
		\boldsymbol{J}^{ii}(\boldsymbol{z})=\begin{bmatrix}\frac{\hat b^i(\tau(\boldsymbol{p}_1^i))}{b^i(\boldsymbol{p}_1^i)}&0&\cdots&0&0\\
			*&\frac{\hat b^i(\tau(\boldsymbol{p}_2^i))}{b^i(\boldsymbol{p}_2^i)}&0&\cdots&0\\
			\vdots&\vdots&\vdots&\ddots&\vdots\\
			*&\cdots&*&\frac{\hat b^i(\tau(\boldsymbol{p}_{n-1}^i))}{b^i(\boldsymbol{p}_{n-1}^i)}&0\\
			*&*&\cdots&*&\frac{\hat b^i(\tau(\boldsymbol{p}_n^i))}{b^i(\boldsymbol{p}_n^i)}\end{bmatrix},
		\boldsymbol{J}^{ij}(\boldsymbol{z})=\begin{bmatrix}0&0&\cdots&0&0\\
			*&0&0&\cdots&0\\
			\vdots&\vdots&\vdots&\ddots&\vdots\\
			*&\cdots&*&0&0\\
			*&*&\cdots&*&0\end{bmatrix}.
	\end{equation}
	Then from (\ref{omega d z multi}), after some calculation, we obtain
	\begin{equation}
		\boldsymbol{J}^{ii}(\boldsymbol{z})=\begin{bmatrix}\frac{\hat b^i(\tau(\boldsymbol{p}_1^i))}{b^i(\boldsymbol{p}_1^i)}&0&\cdots&0&0\\
			0&\frac{\hat b^i(\tau(\boldsymbol{p}_2^i))}{b^i(\boldsymbol{p}_2^i)}&0&\cdots&0\\
			\vdots&\vdots&\vdots&\ddots&\vdots\\
			0&\cdots&0&\frac{\hat b^i(\tau(\boldsymbol{p}_{n-1}^i))}{b^i(\boldsymbol{p}_{n-1}^i)}&0\\
			0&0&\cdots&0&\frac{\hat b^i(\tau(\boldsymbol{p}_n^1))}{b^i(\boldsymbol{p}_n^1)}\end{bmatrix},
		\boldsymbol{J}^{ij}(\boldsymbol{z})=\begin{bmatrix}0&\cdots&0\\ 
			\vdots&\ddots&\vdots\\ 
			0&\cdots&0\end{bmatrix}.
	\end{equation}
	Hence 
	\begin{equation}
		J(z)=\mbox{diag} \{\frac{\hat b^1(\tau(\boldsymbol{p}^1_1))}{b^1(\boldsymbol{p}^1_1)}, \cdots, \frac{\hat b^1(\tau(\boldsymbol{p}^1_n))}{b^1(\boldsymbol{p}^1_n)},\cdots,\frac{\hat b^d(\tau(\boldsymbol{p}^d_1))}{b^d(\boldsymbol{p}^d_1)}, \cdots, \frac{\hat b^d(\tau(\boldsymbol{p}^d_n))}{b^d(\boldsymbol{p}^d_n)}\}.
		\label{Jz diag multi}
	\end{equation}
	Then from (\ref{omega d z multi}) and (\ref{temp multi}), by following step 4 in the proof of Theorem \ref{Theorem affine nonlinear single-input}, we have
	\begin{equation}
		\boldsymbol{J}(\boldsymbol{z})=\mbox{diag}\{\Delta b^1 \boldsymbol{I}_{n}, \cdots, \Delta b^d \boldsymbol{I}_{n}\}. 
	\end{equation} 
	Finally, under the assumption that the equation $\boldsymbol{c}=\boldsymbol{\bar A}(\boldsymbol{c})\boldsymbol{c}$ has the unique solution $\boldsymbol{c}=\boldsymbol{0}$, by using similar procedures as in the proof of Theorem \ref{Theorem affine nonlinear single-input}, we can easily derive all the conclusions in Theorem \ref{Theorem affine nonlinear multi-input}.
	
\end{proof}

\section{Results}
  In this section, we provide the experimental results to verify our theorems. We use synthetic datasets for both linear and affine nonlinear systems.
 \subsection{Experimental Setup}
 \label{Experimental setup}
 \paragraph{Data Generation} For linear systems, we generate datasets by following the data generating process (DGP) described in (\ref{linear system with multi-input}) and (\ref{g}). The coefficients in the dynamic model (\ref{linear system with multi-input}) are randomly chosen, where $\forall i,j = 1,\cdots,d$ and $\forall k=1, \cdots n$, $a_k^{ij}$ is chosen from $U(-2,2)$ and $b^i$ is chosen from $U(0.2,2)$. By following \cite{zimmermann2021contrastive, ahuja2022weakly}, we use a randomly initialized invertible multi-layer perceptron (MLP) with two hidden layers to approximate the mixing function (\ref{g}). For affine nonlinear systems, we generate datasets with DGP (\ref{nonlinear system with multi-input}) and (\ref{g}). We use 3-layer randomly initialized MLPs to approximate the nonlinear mechanisms. The mixing function is the same as in the linear case. For each system, we first consider the noiseless case; then we add white Gaussian noises to each subsystem.

 We use two ways to collect data. \textbf{Passively collected data}: we let each component of the initial state $\boldsymbol{z}^0$ follow an uniform distribution $U(-1,1)$. Each component of $\boldsymbol{u}^0$ has $\frac{1}{2}$ chance to be set to zero and $\frac{1}{2}$ chance to follow an uniform distribution $U(-1,1)$. Then we collect $\boldsymbol{x}^0$, $\boldsymbol{u}^0$, and next-step observation $\boldsymbol{x}^1$. \textbf{Actively collected data}: we fix the initial state $\boldsymbol{z}^0$ to the zero point and let the input $\boldsymbol{u}^t$ at each time step $t$ has the same distribution as $\boldsymbol{u}^0$ in the former case. We simulate the dynamic process up to time step $T_{\max}=5$ and collect $\boldsymbol{x}^t$, $\boldsymbol{u}^t$, and $\boldsymbol{x}^{t+1}$ at each time step $t>0$. 
 
 \paragraph{Model Architectures}
 For linear cases, we use a 3-layer MLP as the representation function $f$ and use linear networks without bias as the linear transformation matrix $\boldsymbol{\bar A}$. $\hat b^i$ is set to 1. For affine nonlinear cases, the representation function $f$ is the same, and we use 3-layer MLPs to model $\boldsymbol{\bar A}(\boldsymbol{z})$ and $\boldsymbol{\bar B}(\boldsymbol{z})$.
 
 \paragraph{Evaluation Metrics}
 We test both the learned representations and the learned dynamic models. To evaluate the learned representations, we use the mean correlation coefficient (MCC)  \cite{hyvarinen2016unsupervised,hyvarinen2017nonlinear,zimmermann2021contrastive} between the learned and ground truth latent variables. Since our theoretical results show that the latents are identified up to scaling, we report MCC based on Pearson (linear) correlation and do not allow permutation when calculating MCC.  
 
 To evaluate the learned dynamic models, for linear systems, we first let $\tilde{a}^{ij}_k=\frac{b^i \hat b^j}{\hat b^i b^j}\hat a^{ij}_k$, $\forall i,j,k$. Then we calculate the parameter estimation error as
 \begin{equation}
 	error=\sqrt{\dfrac{\sum_{i,j,k} (\tilde a^{ij}_k-a^{ij}_k)^2}{\sum_{i,j,k}(a^{ij}_k)^2 }}.
 	\label{error}
 \end{equation}

 For nonlinear systems, we report the MCC between the one-step prediction $\boldsymbol{\tilde z}^{t+1}$  form $\boldsymbol{\hat z}^t$ using the learned dynamic model (see equation (\ref{tilde z^{t+1}})) and the ground truth $\boldsymbol{z}^{t+1}$.
 
 \subsection{Experimental Results} 
 
 \paragraph{Results for Linear Systems}
 
 The experimental results for noiseless linear systems are summarized in Table \ref{table 1}. PAS and ACT stand for passively collected data and actively collected data, respectively, which we detailed in Section \ref{Experimental setup}. Error denotes the parameter estimation error calculated using (\ref{error}). We report the average score from three different randomly generated systems. 
 
 \begin{table}[h]
 	\caption{Results for linear systems without noise.}
 	\label{table 1}
 	\centering
 				\begin{tabular}{ccccc}
 					 \hline
 					$n$&	$d$&DGP  &MCC&Error \\
 					 \hline
 					4 & 1 & PAS &1.000 &0.003 \\ 
 					2 & 2 & PAS &1.000 &0.005 \\
 					3 & 2 & PAS &1.000 &0.017 \\
 					2 & 3 & PAS &0.999 &0.025 \\ \hline
 					4 & 1 & ACT &1.000 &0.031 \\ 
 					2 & 2 & ACT &1.000 &0.032 \\
 					3 & 2 & ACT &0.996 &0.149 \\
 					2 & 3 & ACT &0.999 &0.046 \\
 					 \hline
 				\end{tabular}  
 \end{table}
 
 The results show that the learning method can identify the representations well. The performance for passively collected data is better than the results for actively collected data (with the maximum time step $T_{\max}=5$). One possible explanation is that for some unstable systems, the states diverge rapidly under random inputs, and we need more data to cover these states.  
 
 { To our knowledge, there are almost no previous identifiable representation learning methods fit these deterministic high-order dynamic systems, while these systems are widely used in the system and control theory. The methods using independent components (the first class in Section \ref{sec Identifiable Representation Learning for Dynamic Processes}) are not suitable for our settings, since we do not have independent components they are trying to discover, and they are not aiming to model deterministic difference equations (the first type in Table 1 in \cite{scholkopf2021toward}).
 	
 	 The work most relevant to us is \cite{ahuja2022weakly}, but \cite{ahuja2022weakly} does not explicitly consider dynamic systems. To show the strengths of using high-order dynamic models, we compare our method, which use dynamic model (\ref{linear system with multi-input}), with methods using first-order integrators, which is adopted from the method in \cite{ahuja2022weakly}. We use one specific randomly generated linear system with $n=3$, $d=2$ as the example, and report the blockwise MCC (in the sense of \cite{ahuja2022weakly}) in Table \ref{table 1.1}. 
 
 \begin{table}[h]
 	\caption{ The comparison of using first-order integrators and dynamic model (\ref{linear system with multi-input}).}
 	\label{table 1.1}
 	\centering
 	\begin{tabular}{cccc}
 		\hline
 		Model& blockwise MCC \\
 		\hline
 		first-order integrators & 0.283 \\ 
 		dynamic model (\ref{linear system with multi-input})\textbf{(ours)} & 1.000   \\
 		\hline
 	\end{tabular}  
 \end{table}
   The results show that modelling as first-order integrators can not capture the latents under high-order dynamic mechanisms, which demonstrates the necessity of using the high-order dynamic model in representation learning.}

 We then consider the effect of the noise. We use one specific linear system with $n=3$, $d=2$ as the example and add white Gaussian noises with zero expectations and different standard deviations ($\sigma=0.02,0.20,1.00$) to the system. {The noises are added to the systems in the same channels as the inputs but are not observed.} The initial states and inputs are still range from $-1$ to $1$. The results are shown in Table \ref{table 2}.
 
 \begin{table}[h]
 	\caption{Results for linear systems with noise.}
 	\label{table 2}
 	\centering
 				\begin{tabular}{cccc}
 					 \hline
 					$\sigma$&DGP &MCC&Error\\
 					 \hline
 					0.00 & PAS &1.000 &0.017 \\ 
 					0.02 & PAS &1.000 &0.013 \\ 
 					0.20 & PAS &0.999 &0.034 \\
 					1.00 & PAS &0.917 &0.616 \\ \hline
 					0.00 & ACT &0.993 &0.172 \\
 					0.02 & ACT &0.992 &0.166 \\ 
 					0.20 & ACT &0.991 &0.235 \\ 
 					1.00 & ACT &0.952 &0.457 \\
 					 \hline
 				\end{tabular}  
 \end{table}
 
  Table \ref{table 2} shows that, as the noise increases, the performances of both learned representations and learned models decrease. But relatively small noise has little effect on the results. 
 
 \paragraph{Results for Affine Nonlinear Systems}
 
 We show the results for affine nonlinear systems without and with noise in Table \ref{table 3} and Table \ref{table 4}. MCC(R) and MCC(M) denote the performance of learned representations and learned models, respectively. The results for noiseless cases are the average score from three different randomly generated systems. The results for systems with noise are based on one specific affine nonlinear system with $n=3$ and $d=2$.
 
 \begin{table}[h]
 	\caption{Results for affine nonlinear systems without noise.}
 	\label{table 3}
 	\centering
 				\begin{tabular}{ccccc}
 					 \hline
 					$n$&$d$&DGP&MCC(R)&	MCC(M)\\
 					 \hline
 					4 & 1 & PAS &1.000 &1.000 \\ 
 					2 & 2 & PAS &1.000 &1.000 \\
 					3 & 2 & PAS &0.995 &0.996 \\
 					2 & 3 & PAS &0.996 &0.996 \\ \hline
 					4 & 1 & ACT &1.000 &1.000 \\ 
 					2 & 2 & ACT &1.000 &1.000 \\
 					3 & 2 & ACT &0.995 &0.996 \\
 					2 & 3 & ACT &0.996 &0.996 \\
 					 \hline
 				\end{tabular}   
 \end{table}
 
 \begin{table}[h]
 	\caption{Results for affine nonlinear systems with noise.}
 	\label{table 4}
 	\centering
 				\begin{tabular}{cccc}
 					 \hline
 					$\sigma$&DGP &MCC(R)&MCC(M)\\
 					 \hline
 					0.00 & PAS &0.995 &0.996 \\ 
 					0.02 & PAS &0.996 &0.997 \\
 					0.20 & PAS &0.993 &0.993 \\
 					1.00 & PAS &0.882 &0.727 \\ \hline
 					0.00 & ACT &0.995 &0.997 \\
 					0.02 & ACT &0.987 &0.990 \\ 
 					0.20 & ACT &0.983 &0.959 \\ 
 					1.00 & ACT &0.903 &0.711 \\
 					 \hline
 				\end{tabular}  
 \end{table}
 
 The results show that the method can identify representations and affine nonlinear dynamic models well for noiseless and low-noise systems. When the noise increases, both performances decrease, and the performance for model learning has been more affected.

\section{Discussion}
  In this paper, we study the problem of identifiable representation and model learning for latent dynamic systems. We have shown that, for both linear systems and affine nonlinear systems, controllable canonical forms can be used as an inductive bias to identify the latents up to scaling and determine the dynamic models up to some simple transformations. These results help to understand latent mechanisms behind complex dynamic observations and may lead to more trustworthy decision-making and control methods for intelligent spacecrafts.
  
   However, there are still some limitations. For example, how to learn identifiable representations for general nonlinear systems is still unclear. In this work, our theoretical results only consider deterministic dynamic mechanisms. Hence, a natural extension is considering the case where both stochastic and deterministic mechanisms appear. Our results only consider continuous state space and input space. How to model the deterministic indirect effect of discrete input is still unclear.

\bibliographystyle{IEEEtran}
\bibliography{./sample.bib} 

\end{document}